\documentclass{article}
\pdfoutput=1

\usepackage[final,nonatbib]{neurips_2021}

\usepackage{cancel}
\usepackage[utf8]{inputenc} % allow utf-8 input
\usepackage[T1]{fontenc}    % use 8-bit T1 fonts
\usepackage{hyperref}       % hyperlinks
\usepackage{url}            % simple URL typesetting
\usepackage{booktabs}       % professional-quality tables
\usepackage{amsmath}
\usepackage{amsfonts}       % blackboard math symbols
\usepackage{nicefrac}       % compact symbols for 1/2, etc.
\usepackage{microtype}      % microtypography
\usepackage{xcolor}         % colors
\usepackage{graphicx}
\usepackage{subcaption}
\usepackage{floatrow}
\usepackage{placeins}
\usepackage{blindtext}
\usepackage{algorithm}
\usepackage{algpseudocode}

\usepackage{amsmath,amsfonts,bm,amssymb,amsthm,yfonts}

\def\eqref#1{equation~\ref{#1}}
\def\1{\bm{1}}

\def\rt{{\textnormal{t}}}

\def\vx{{\bm{x}}}

\DeclareMathAlphabet{\mathsfit}{\encodingdefault}{\sfdefault}{m}{sl}
\SetMathAlphabet{\mathsfit}{bold}{\encodingdefault}{\sfdefault}{bx}{n}

\def\eee#1#2{\mathbb{E}_{#1}\left[#2\right]}

\newtheorem{theorem}{Theorem}
\newtheorem{definition}{Definition}
\newtheorem{lemma}{Lemma}
\newtheorem{corollary}{Corollary}

\newcommand\Ccancel[2][black]{\renewcommand\CancelColor{\color{#1}}\cancel{#2}}

\definecolor{darkgreen}{RGB}{0,150,0}
\def\vx{\mathbf{x}}

\newcommand\yuandong[1]{\textcolor{orange}{Yuandong: #1}}
\newcommand\tianjun[1]{\textcolor{red}{Tianjun: #1}}
\newcommand\kevin[1]{\textcolor{blue}{Kevin: #1}}
\newcommand\chris[1]{\textcolor{purple}{Chris: #1}}
\newcommand\benoit[1]{\textcolor{cyan}{Benoit: #1}}
\newcommand\brandon[1]{\textcolor{violet}{Brandon: #1}}

\newcommand{\etal}{{et al}.\@ }

\def\ours{\texttt{LaP$^3$}}

\title{Learning Space Partitions for Path Planning}

\author{%
Kevin Yang$^{1*}$ \quad Tianjun Zhang$^{1*}$ \quad Chris Cummins$^2$ \quad Brandon Cui$^2$ \quad Benoit Steiner$^2$ \\
\textbf{Linnan Wang}$^3$ \quad \textbf{Joseph E. Gonzalez}$^1$ \quad \textbf{Dan Klein}$^1$ \quad \textbf{Yuandong Tian}$^2$\\
$^1$UC Berkeley \quad $^2$Facebook AI Research \quad $^3$Brown University\\
\texttt{\{yangk,tianjunz,jegonzal,klein\}@berkeley.edu}\\
\texttt{\{cummins,bcui,benoitsteiner,yuandong\}@fb.com}\\
\texttt{linnan\_wang@brown.edu}
}

\begin{document}

\maketitle

\iffalse
\begin{abstract}
To optimize a high-dimensional function $f(\vx)$ on which a global model is hard to build, modern optimization techniques~\cite{} often partition the search space into sub-regions and focus on promising ones. While the algorithm of region partition is manually defined in most works (e.g., DOO~\cite{}, VOO~\cite{}), recent works like LaNAS~\cite{wang2021sample} and LAMCTS~\cite{wang2020learning} empirically show that using machine learning models (e.g., classifiers) trained from existing samples of $f$ to dynamically partition the region could yield higher sample efficiency. In this paper, we give a theoretical explanation of LaMCTS' effectiveness and how region splitting aided by ML models leads to a better regret bound. Based on our analysis, we propose \ours{} to address the weakness of LaMCTS, by learning representation of the search space during the optimization, and improving the function value estimation at each sub-region. We apply \ours{} to path planning, a family of difficult high-dimensional optimization problems that often appear in many real-world scenarios. \ours{} substantially outperforms existing path planning methods like Random Shooting, CEM, CMA-ES and LaMCTS in a range of 2D navigation tasks containing difficult-to-escape local optima, and these gains transfer to real-world tasks: we outperform strong baselines (OpenTuner) in compiler phase ordering by up to $19.10\%$ and in molecular design by up to 0.4 on properties on a 0-1 scale.
% \yuandong{Some details about the experimental results}
\end{abstract}
\fi

\begin{abstract}
Path planning, the problem of efficiently discovering high-reward trajectories, often requires optimizing a high-dimensional and multimodal reward function. Popular approaches like CEM~\cite{cem} and CMA-ES~\cite{hansen2016cma} greedily focus on promising regions of the search space and may get trapped in local maxima. DOO~\cite{doo} and VOOT~\cite{voot} balance exploration and exploitation, but use space partitioning strategies independent of the reward function to be optimized. Recently, LaMCTS~\cite{wang2020learning} 
empirically learns to partition the search space in a reward-sensitive manner for black-box optimization. In this paper, we develop a novel formal regret analysis for when and why such an adaptive region partitioning scheme works. We also propose a new path planning method \ours{} which improves the function value estimation within each sub-region, and 
uses a latent representation of the search space.
Empirically, \ours{} outperforms existing path planning methods in 2D navigation tasks, especially in the presence of difficult-to-escape local optima, and shows benefits when plugged into the planning components of model-based RL such as PETS~\cite{chua2018deep}. These gains transfer to highly multimodal real-world tasks, where we outperform strong baselines in compiler phase ordering by up to 39\% on average across 9 tasks, and in molecular design by up to 0.4 on properties on a 0-1 scale. Code is available at \url{https://github.com/yangkevin2/neurips2021-lap3}.  
\end{abstract}

\iffalse
\begin{abstract}
While optimization problems can be formulated into Markov Decision Processes (MDP) and addressed by reinforcement learning (RL) techniques, one optimization problem may correspond to many different MDPs, some of which can be addressed much more easily than others. Furthermore, given an optimization problem, recent works (e.g., LaMCTS) take it into a meta-level by finding good action spaces (and thus MDPs) by recursive partition of search space on the fly. In this paper, we extend LaMCTS in multiple ways. First, we provide a regret analysis and a theoretical explanation why it works and how space splitting leads to better regret bound. Second, we improve LaMCTS based on the theoretical findings, and apply it to path planning, a difficult high-dimensional optimization problem that is extensive used in many real-world scenarios. We show that our approach, \ours{}, outperforms existing path planning methods in compiler phase ordering and molecule discovery.

\kevin{Actually, my group meeting this week got canceled last-minute, so NeurIPS feedback will happen next week. But anyway my own comments-- the first sentence was a bit confusing to me, do we need it? My proposed edit in the latex.}
\benoit{I agree that the new wording is much clearer}

recursively split the space  good action spaces---and thus Markov Decision Processes (MDPs)---by recursively partitioning the search space on the fly. 
\end{abstract}
\fi

\section{Introduction}

\iffalse
can be roughly put into two categories: gradient-based and gradient-free approaches . Both approaches have pros and cons: gradient-based approach focuses on local improvement, is sample efficient if the reward landscape is small, but may not be able to escape from local optimal solutions. On the other hand, the gradient-free approach keeps a big picture of the landscape but may not be sample efficient in precisely finding the exact optimal solution.  
\fi

\iffalse
Local approaches such as gradient descent searches locally and find a solution, which is often fast but might get trapped into a local minima. 
\fi

Path planning has been used extensively in many applications, ranging from reinforcement learning~\cite{chua2018deep,hafner2019learning,hafner2020mastering} and robotics~\cite{mac2016heuristic,ratliff2009chomp,li2004iterative} to biology~\cite{kroer2016sequential}, chemistry~\cite{segler2018planning}, material design~\cite{kajita2020autonomous}, and compiler optimization~\cite{triantafyllis2003compiler}. The goal is to find the most rewarding trajectory (i.e., state-action sequence) $\vx = (s_0, a_0, s_1, \ldots, s_n)$ in the search space $\Omega$: $\vx^* = \arg\max_{\vx\in\Omega} f(\vx)$, where $f(\vx)$ is the reward.

In this work, we focus on deterministic path planning problems with long trajectories $\vx$, and discontinuous and/or multimodal reward functions $f$. Such high-dimensional non-convex optimization problems exist in many real domains, both continuous and discrete. While we could always find near-optimal $\vx$ by random sampling given an infinite query budget, in practice we prefer a sample-efficient method that achieves high-reward trajectories with fewer queries of the reward function $f$.  

\iffalse
For many real problems, the trajectory $\vx$ can be long, the states/actions can be either continuous or discrete, and function $f$ can be discontinuous and multi-modal, resulting in a difficult high-dimensional non-convex optimization problem. While we could always find near-optimal $\vx$ by random sampling if given an infinite query budget, in practice it is preferred to have a sample-efficient method that achieves high-reward trajectories with the fewest queries of the reward function $f$. 
\fi

While global methods like Bayesian Optimization (BO)~\cite{brochu2010tutorial} may struggle with limited samples and high-dimensional spaces, classic approaches like CEM~\cite{cem} and CMA-ES~\cite{hansen2016cma} learn a local model around promising trajectories. For example, CEM tracks a population of trajectories and repeatedly re-samples its population according to the highest-performing trajectories from the previous generation. On the other hand, such a focus can trap CEM in local optima, as confirmed empirically (Sec. \ref{sec:synthetic_experiments}).

Other recent approaches, such as VOOT~\cite{voot} and DOO~\cite{doo}, use a (recursive) region partitioning scheme: they split the search space $\Omega$ into sub-regions $\Omega=\Omega_1 \cup \ldots \cup \Omega_k$, then invest more samples into promising sub-regions while continuing to explore other regions via an \emph{upper confidence bound} (UCB). While such exploration-exploitation procedures adaptively focus on promising sub-regions and lead to sub-linear regret and optimality guarantees, their \emph{region partition} procedure is manually designed by humans and remains non-adaptive. For example, DOO partitions the space with uniform axis-aligned grids and VOOT with Voronoi cells, both independent of the reward $f$ to be optimized. 

%VOOT can then focus more samples into promising subtrees, and provides some theoretical guarantees on regret.
% Here ``promising'' is often estimated by previously sampled data within the sub-region, plus a confidence bound that is related to the number of samples in the region. 
% Some approaches, e.g., DOO~\cite{}, also provide theoretical guarantees. 

%While such approaches adaptively choose which , 
Recently, Wang \etal proposed LaNAS~\cite{wang2021sample} and LaMCTS~\cite{wang2020learning}, which \emph{adaptively} partition the search regions based on sampled function values, and focus on good regions. They achieve strong empirical performance on Neural Architecture Search (NAS) and black-box optimization, outperforming many existing methods including evolutionary algorithms and BO. Notably, in recent NeurIPS'20 black-box optimization challenges, two teams that use variants of LaMCTS ranked 3rd~\cite{sazanovich2020solving} and 8th~\cite{kim2020adaptive}. 

In this paper, we provide a simple theoretical analysis of LaMCTS to reveal the underlying principles of adaptive region partitioning, an analysis missing in the original work. Based on this analysis, we propose \textbf{La}tent Space \textbf{P}artitions for \textbf{P}ath \textbf{P}lanning (\ours{}), a novel optimization technique for path-planning. 
%Compared to LaMCTS, which does not use a latent space, \ours{} uses two separate latent representations of the original search space, disentangling the \textit{partitioning} space for region partitioning from the \textit{sampling} space for proposing new trajectories at tree leaves. 
Unlike LaMCTS, \ours{} uses a latent representation of the search space. 
Additionally, we use the maximum (instead of the mean) as the node score to improve sample efficiency, verified empirically in Sec. \ref{sec:analysis}. Both changes are motivated by our theoretical analysis. 

We verify \ours{} on several challenging path-planning tasks, including 
%\textbf{(a)} 
2D navigation environments from past work with difficult-to-escape local optima,
%(MiniWorld~\cite{gym_miniworld} modified for a continuous action space, and existing discrete tasks from MiniGrid~\cite{gym_minigrid}), 
and 
%\textbf{(b)} 
real-world planning problems in compiler optimization and molecular design. In all tasks, \ours{} demonstrates substantially stronger exploration ability to escape from local optima compared to several baselines including CEM, CMA-ES and VOOT. On compiler phase ordering, we achieve on average 39\% and 31\% speedup in execution cycles comparing to -O3 optimization and OpenTuner~\cite{ansel2014opentuner}, two widely used optimization techniques in compilers. On molecular design, \ours{} outperforms all of our baselines in generating molecules with high values of desirable properties, beating the best baseline in average property value by up to $0.4$ on properties in a $[0,1]$ range. Additionally, extensive ablation studies show factors that affect the quality of planning and verify the theoretical analysis. 

\ours{} is a general planning technique and can be readily plugged into existing algorithms with path planning components. For example, we apply \ours{} to PETS~\cite{chua2018deep} in model-based RL and observe substantially improved performance for high-dimensional continuous control and navigation, compared to CEM as used in the original PETS framework.

\def\eee#1#2{\mathbb{E}_{#1}\left[#2\right]}
\def\ee#1{\mathbb{E}\left[#1\right]}

\def\cT{\mathcal{T}}
\def\vol{\mathrm{vol}}
\def\cO{\mathcal{O}}
\def\kgood{\mathcal{K}_{\good}}
\def\kbad{\mathcal{K}_{\bad}}
\def\cV{\mathcal{V}}
\def\cS{\mathcal{S}}
\def\cD{\mathcal{D}}
\def\rt{\mathrm{root}}
\def\leaf{\mathrm{leaf}}
\def\child{\mathrm{child}}

\def\init{\mathrm{init}}
\def\good{\mathrm{good}}
\def\bad{\mathrm{bad}}
\def\parent{\mathrm{parent}}
\def\partition{\mathrm{par}}
\def\thres{\mathrm{thres}}

\vspace{-0.2em}
\section{Latent Space Monte Carlo Tree Search (LaMCTS)}

\begin{algorithm}[b]
    \small
	\caption{\ours{} Pseudocode for Path Planning. Improvements over LaMCTS in \textcolor{darkgreen}{green}.}
\label{alg:plalam}
	\begin{algorithmic}[1]
	\State {\bfseries Input:} Number of rounds $T$, Environment Oracle: $f(\vx)$, \textcolor{darkgreen}{Dataset $\cD$, Sampling Latent Model $h(\vx)$, Partitioning Latent Model $s(\vx)$.}
	\State {\bfseries Parameters:} Initial \#samples $N_\init$, Re-partitioning interval $N_{\partition}$, Node partition threshold $N_\thres$, UCB parameter $C_p$.
	\State \textcolor{darkgreen}{Pre-train $h(\cdot)$ on $\cD$ when $\cD \neq \emptyset$}. 
	\State Set region partition $\cV_0 = \{ \Omega \}$. 
	\State Draw $N_\init$ samples uniformly from $\cS_0 = \{(\vx_i, f(\vx_i))\}_{i=1}^{N_\init} \subset \Omega$. 
	\For{$t = 0, \dots, T-N_\init-1$}
	\If{$t$ divides $N_\partition$}
	\State \textcolor{darkgreen}{Train/fine-tune latent model $h(\cdot)$ using samples $\cS_t \cup \cD$ (Eqn.~\ref{eq:latent-training})}.   
	\State Re-learn region partition $\cV_t \leftarrow \mathrm{Partition}(\Omega, \cS_t, N_{\thres}, \textcolor{darkgreen}{s(\cdot)})$ \textcolor{darkgreen}{in latent space $\Phi_s$ of $s(\cdot)$}.
	\EndIf
	\For{$k := \rt$, $k \notin \cV_\leaf$}
	    \State $k \leftarrow \displaystyle\arg\max_{\Omega_c \in \child(\Omega_{k})} b_c$, where $b_c := \left[\Ccancel[darkgreen]{\displaystyle\frac{1}{n(\Omega_c)}\displaystyle\sum_{\vx_i \in \Omega_c} f(\vx_i)}\hspace{0.3em} \textcolor{darkgreen}{\displaystyle\max_{\vx_i \in \Omega_c} f(\vx_i)} + C_p \sqrt{\frac{2\log n(\Omega_{k})}{n(\Omega_c)}  }\right]$. 
	\EndFor
	\State \textcolor{darkgreen}{Initialize CMA-ES using encodings of $\cS_t \cap \Omega_k$ via $h(\cdot)$.} Here $\Omega_k$ is the chosen leaf sub-region.
	\State $\cS_t \leftarrow \cS_{t-1} \cup \{(\vx_t, f(\vx_t))\}$, where $\vx_t$ is \textcolor{darkgreen}{drawn from CMA-ES and decoded via $h^{-1}(\cdot)$.}
	%from $\Omega_{k}$ \textcolor{darkgreen}{in latent space of $h(\cdot)$}. 
% 	\State If $n(\Omega_k) \ge N_\thres$, then split region $\cV_t \leftarrow \cV_{t-1} \cup \mathrm{Partition}(\Omega_k, \cS_t \cap \Omega_k, \textcolor{darkgreen}{s(\cdot)})$.  
	\EndFor
   \end{algorithmic}
\end{algorithm}

\begin{figure}[bt]
    \centering
    \includegraphics[width=\textwidth]{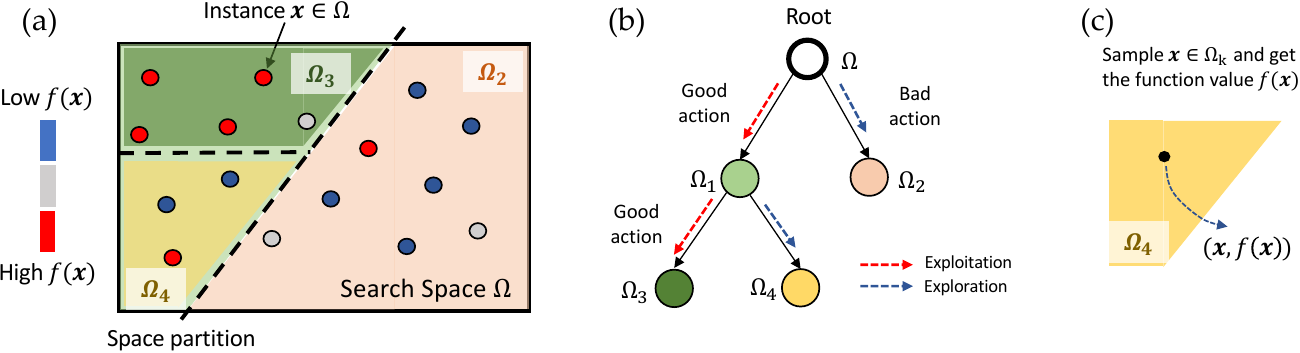}
    \vspace{-0.2in}
    \caption{\small \ours{} extends LaMCTS~\cite{wang2020learning} to path planning. \textbf{(a)} Starting from a search space $\Omega$, both \ours{} and LaMCTS first draw a few samples $\vx\in \Omega$, then learn to partition $\Omega$ into a sub-region $\Omega_1$ with good samples (high $f(\vx)$) and a sub-region $\Omega_2$ with bad samples (low $f(\vx)$). Compared to LaMCTS, \ours{} uses a \emph{latent space} and reduces the dimensionality of the search space. \textbf{(b)} Sampling follows the learned recursive space partition, focusing on good regions while still exploring bad regions using UCB. \ours{} uses the \emph{maximum} of the sampled value in a region ($\max_{\vx_i\in\Omega} f(\vx_i)$) as the node value, while LaMCTS uses the mean. \textbf{(c)} Upon reaching a leaf, new data points are sampled within the region and the space partition is relearned.}
    \label{fig:overview}
\end{figure}

LaMCTS~\cite{wang2020learning} is recently proposed to solve black-box optimization problems $\vx^* = \arg\max_\vx f(\vx)$ via recursively learning $f$-dependent region partitions. Fig.~\ref{fig:overview} and Alg.~\ref{alg:plalam} show the details of LaMCTS as well as our proposed approach \ours{} (formally introduced in Sec.~\ref{sec:method}) for comparison. 

LaMCTS starts with $N_\init$ random samples of the entire search space $\Omega$ (line 5 in Alg.~\ref{alg:plalam}). For a region $\Omega_k$, let $n(\Omega_k)$ be the number of samples within. LaMCTS dictates that, if $n(\Omega_k) \ge N_\thres$, then $\Omega_k$ is partitioned into disjoint sub-regions $\Omega_k=\Omega_\good\cup \Omega_\bad$ as its \emph{children} (Fig.~\ref{fig:overview}(a)-(b), line 9 in Alg.~\ref{alg:plalam}, the function $\mathrm{Partition}$). Intuitively, $\Omega_\good$ contains promising samples with high $f$, while $\Omega_\bad$ contains samples with low $f$. Unlike DOO and VOOT, such a partition is learned using $\cS_t\cap \Omega_k$, our samples so far in the region, and is thus dependent on the function $f$ to be optimized. 

Given tree-structured sub-regions, new samples are mostly drawn from promising regions and occasionally from other regions for exploration. This is achieved by Monte Carlo Tree Search (MCTS)~\cite{browne2012survey} (line 11-13): at each tree branching, the UCB score $b$ is computed to balance exploration and exploitation (line 12).  
Then the subregion with highest UCB score is selected (e.g., it may have high $f$ and/or low $n$). This is done recursively until a leaf sub-region $\Omega'$ is reached. Then a new sample $\vx$ is drawn from $\Omega'$ (line 15) either uniformly, or from a local model constructed by an existing optimizer (e.g., TuRBO~\cite{eriksson2019scalable}, CMA-ES~\cite{hansen2016cma}), in which case LaMCTS becomes a meta-algorithm. When more samples are collected, regions are further partitioned and the tree gets deeper.

Finally, the function $\mathrm{Partition}$ in Alg.~\ref{alg:plalam} is defined as follows: 
first a 2-class K-means on $(\vx, f(\vx))$ is used to create positive/negative sample groups. Next, a SVM classifier is used to learn the decision boundary (hence the partition), so that samples with high $f(\vx)$ fall into $\Omega_\good$, and samples with low $f(\vx)$ fall into $\Omega_\bad$ (Fig.~\ref{fig:overview}(a)). See Appendix \ref{sec:alg_partition} for the pseudo code. The partition boundary can also be re-learned after more samples are collected (line 9).

\iffalse
Its main idea is: for each (sub)-region $\Omega_k$, learn $f$-dependent partition of the search space based on the sampled data   

and guide the search towards a promising region of the space, while keeping a balance between exploitation (i.e., further exploring high-performing regions) and exploration (i.e., occasionally checking low-performing regions to see whether there are high-performing sub-regions).
\fi

%\yuandong{Put an algorithmic block here to make sure the audience clearly see what is LaMCTS and how our approach differs from it. }

\iffalse
The goal here is to find a strategy to maximize the overall quality of the $n$ samples, or more precise, to minimize the regret. If we already know the quality of the samples, then obviously we could find a perfect split of $\Omega_+$ and $\Omega_-$, so that $\Omega_+$ only contains samples of high values. The regret is thus the difference between the perfect split and the split we can achieve algorithmically.  
\fi

\vspace{-0.1in}
\section{A Theoretical Understanding of Space Partitioning}\label{sec:theory}
\vspace{-0.05in}
While LaMCTS~\cite{wang2020learning} shows strong empirical performance, it contains several components with no clear theoretical justification. Here we attempt to give a formal regret analysis when sub-regions $\{\Omega_k\}$ are \emph{fixed} and all at the same tree level, and the function $f$ is deterministic. We leave further analysis of tree node splitting and evolution of hierarchical structure to future work. 

Despite the drastic simplification, our regret bound still shows why an $f$-dependent region partition is helpful. By showing that a better regret bound can be achieved by a clever region partition as empirically used in the $\mathrm{Partition}$ function in Alg.~\ref{alg:plalam}, we justify the design of LaMCTS. Furthermore, our analysis suggests several empirical improvements over LaMCTS and motivates the design of \ours{}, which outperforms multiple classic approaches on hard path planning problems.

\def\pr{\mathbb{P}}

\begin{figure}
    \centering
    \includegraphics[width=\textwidth]{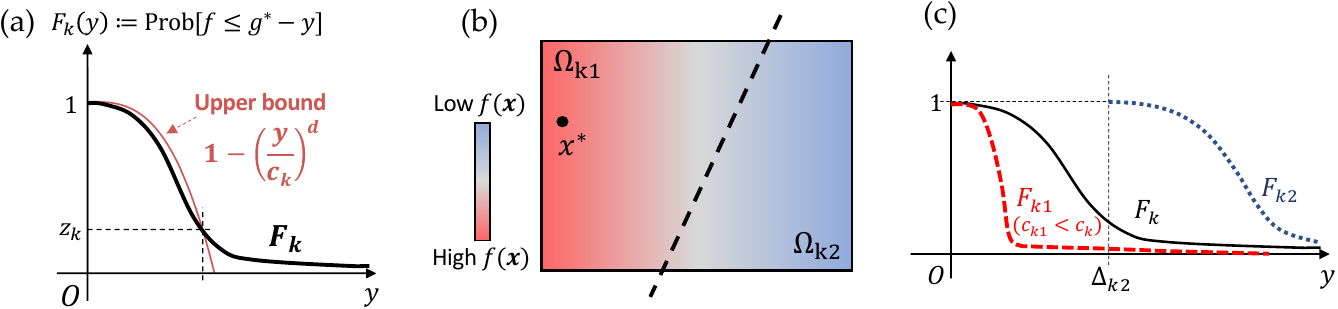}
    \vspace{-0.2in}
    \caption{\small Theoretical understanding of space partitioning. \textbf{(a)} Definition of $(z_k, c_k)$-diluted region $\Omega_k$ (Def.~\ref{def:diluted}). \textbf{(b)} Partition of region $\Omega_k$ into good region $\Omega_{k1}$ and bad region $\Omega_{k2}$. Optimal solution $\vx^* \in \Omega_{k1}$. \textbf{(c)} After space partitioning, $F_k$ is split into $F_{k1}$ and $F_{k2}$. The good region $F_{k1}$ has much smaller $c_{k1}$ while the bad region has much larger best-to-optimality gap $\Delta_{k2}$. As a result, the expected total regret decreases.} 
    \label{fig:theory}
\end{figure}

\vspace{-0.05in}
\subsection{Regret Analysis with Fixed Sub-Regions}
% \vspace{-0.05in}

We consider the following setting. Suppose we have $K$ $d$-dimensional regions $\{\Omega_k\}_{k=1}^K$, and $n_t(\Omega_k)$ is the visitation count at iteration $t$. The global optimum $\vx^*$ resides in some unknown region $\Omega_{k^*}$. At each iteration $t$, we visit a region $\Omega_k$, sample (uniformly or otherwise) a data point $\vx_t \in \Omega_k$, and retrieve its \emph{deterministic} function value $f_t = f(\vx_t)$. In each region $\Omega_k$, define $\vx^*_k := \arg\max_{\vx\in\Omega_k} f(\vx)$ and the maximal value $g^*(\Omega_k) = f(\vx^*_k)$. The maximal value \emph{so far} at iteration $t$ is $g_t(\Omega_k) = \max_{t'\le t} f(\vx_{t'})$. It is clear that $g_t \le g^*$ and $g_t\rightarrow g^*$ when $t \rightarrow +\infty$.

We define the \emph{confidence bound} $r_t = r_t(\Omega_k)$ so that with high probability, the following holds: 
\begin{equation}
    g_t(\Omega_k) \ge g^*(\Omega_k) -  r_t(\Omega_k) \label{eq:lower-bound}
\end{equation}
At iteration $t$, we pick region $k_t$ to sample based on the upper confidence bound: $k_t = \arg\max_{k} g_t(\Omega_k) + r_t(\Omega_k)$. Many different confidence bounds can be applied; for convenience in this analysis, we use the ``ground truth'' bound from the cumulative density function (CDF) of $f$ within the region $\Omega_k$ (Please check Appendix \ref{sec:proofs} for all proofs): 
\begin{lemma}
\label{thm:max-bound}
    Let $F_{k}(y) := \pr\left[f(\vx)\le g^*(\Omega_k) - y|\vx\in\Omega_k\right]$ be a strictly decreasing function, and let $r_{k,t}(\Omega_k) := F_k^{-1}\left(\delta^{1/n_t(\Omega_k)}\right)$. Then Eqn.~\ref{eq:lower-bound} holds with probability $1 - \delta$.
\end{lemma}
Here $F^{-1}_k$ is the inverse function of $F_k$ and randomness arises from sampling within $\Omega_k$. Since $F_k$ is a strictly decreasing function, $F^{-1}_k$ exists and is also strictly decreasing. By definition, $F_k\in [0,1]$, $F_k(0) = 1$ and $F_k^{-1}(1)=0$. We then define the \emph{dilution} of each region as follows:

\iffalse\kevin{I only skimmed the proof, but it seems like the proof is assuming that the ball around $x_k^*$ is entirely contained in the region? Is this valid? For example if you allow nonconvex regions, I can contrive a very bad 2D example: Your region is $\{|x_1| < \epsilon, 0 < x_2 < 10\} \cup \{|x_1| < 1000, -10 < x_2 < 0\}$ and $f(x) = x_2$. Is the Lemma only supposed to hold for sufficiently large $n_t$?} 

Note that each region $\Omega_k$ has its own Lipschitz constant $L_k$ and we can define an \emph{average} Lipschitz constant $\bar L := \sqrt[d]{\frac{1}{V}\sum_{k=1}^K L_k^d V_k}$, where $V := \sum_k V_k$ is the total volume and is \emph{independent} of the specific partition of the sub-regions $\{\Omega_k\}$. Intuitively, $\bar L^d$ is an average of $\{L_k^d\}$ weighted by volume $V_k$ of each region.  
\fi

\begin{definition}[$(z_k,c_k)$-dilution]
\label{def:diluted}
A region $\Omega_k$ is \emph{$(z_k, c_k)$-diluted} if there exist $z_k, c_k$ such that $F_k(y) \le 1 - (y/c_k)^d$ for $y \in [0,c_k(1-z_k)^{1/d}]$, where $z_k$ is the smallest $F_k(y)$ to make the inequality hold.  
\end{definition}

The intuition for dilution for a given region, as depicted in Fig.~\ref{fig:theory}(a), is that all but $z_k$ fraction of the region has function value close to the maximum, with "close" defined based on $c_k$ (smaller $c_k$ implies a stricter definition of ``close''). %A less diluted region (i.e., $z_k$ and $c_k$ are both small) means that function values in most of the region are concentrated near the maximum, and thus easier to optimize.
% See Fig.~\ref{fig:theory}(a) for visualization. 
Obviously if $\Omega_k$ is $(z_k, c_k)$-diluted then it is $(z_k', c'_k)$-diluted for any $c'_k \ge c_k$ and $z_k' \ge z_k$. Therefore, we often look for the smallest $z_k$ and $c_k$ to satisfy the condition. If a region $\Omega_k$ has small $c_k$ and $z_k$, we say it is \emph{highly concentrated}. For example, if $f(\vx)$ is mostly constant within a region, then $c_k$ is very small since $F_k(y)$ drops to $0$ very quickly. In such a case, most of the region's function values are concentrated near the maximum, making it easier to optimize. 

While the definition of concentration may be abstract, we show it is implied by Lipschitz continuity:
\begin{corollary} 
\label{corollary:lipschitz}
If a region $\Omega_k$ is $L_k$-Lipschitz continuous, i.e., $|f(\vx) - f(\vx')| \le L_k\|\vx-\vx'\|_2$, and there exists an $\epsilon_0$-ball $B(\vx^*_k, \epsilon_0)\subseteq \Omega_k$, then with uniform sampling, $\Omega_k$ is $(1 - \epsilon^d_0 \tilde V^{-1}_k, L_k\sqrt[d]{\tilde V_k})$-diluted. Here $\tilde V_k := V_k / V_0$ is the relative volume with respect to the unit sphere volume $V_0$.
\end{corollary}
Typically, a smoother function (with small $L_k$) and large $\epsilon_0$ yield a less diluted (and more concentrated) region. However, the concept of dilution (Def.~\ref{def:diluted}) is much broader. For example, if we shuffle function values within $\Omega_k$, Lipschitz continuity is likely to break but Def.~\ref{def:diluted} still holds.

Now we will bound the total regret. Let $R_t(a_t) := f^* - g_t(\Omega_{a_t}) \ge 0$ be the regret of picking $\Omega_{a_t}$ and $R(T) := \sum_{t=1}^T R_t(a_t)$ be the total regret, where $T$ is the total number of samples (queries to $f$). Define the \textit{gap} of each region $\Delta_k := f^* - g^*(\Omega_k)$ and split the region indices into $\kgood := \{k: \Delta_k \le \Delta_0\}$ and $\kbad := \{k: \Delta_k \ge \Delta_0\}$ by a threshold $\Delta_0$. $C_\good := \left(\sum_{k \in \kgood} c^d_k\right)^{1/d}$ and $C_\bad := \left(\sum_{k \in \kbad} c^d_k\right)^{1/d}$ are the $\ell_d$-norms of the $c_k$ in these two sets. Finally, $M := \sup_{\vx\in\Omega} f(\vx) - \inf_{\vx\in\Omega} f(\vx)$ is the maximal gap between function values. Treating each region $\Omega_k$ as an arm and applying a regret analysis similar to multi-arm bandits~\cite{slivkins2019introduction}, we obtain the following theorem:% (here $T$ is the total number of samples, or queries of the function $f$):
\begin{theorem}
\label{thm:regret}
Suppose all $\{\Omega_k\}$ are $(z_k, c_k)$-diluted with $z_k\le \eta/T^3$ for some $\eta > 0$. %\kevin{Is this $z_k\le \eta/T^3$ condition a problem? doesn't it mean that the assumption will no longer hold as $T\rightarrow \infty$?}
The total expected regret $\ee{R(T)} = \cO\left[C_\good  \sqrt[d]{T^{d-1}\ln T} +  M (C_\bad / \Delta_0)^d \ln T +  KM\eta /T\right]$.   
\end{theorem}

%\kevin{Can refer to $z_k$ empirical analysis in Appendix \ref{sec:empirical_zk} as needed}

% \vspace{-0.1em}
\subsection{Implications of Theorem~\ref{thm:regret}} \label{sec:implications}
%\kevin{Do we think some of the following three paragraphs could be cut or shortened?}
%\textbf{Sampling within $\Omega_k$.} While  
\textbf{The effect of space partitioning}. Reducing $\{c_k\}$ results in a smaller regret $R(T)$. Thus if we can partition $\Omega_k$ into two sub-regions $\Omega_{k1}$ and $\Omega_{k2}$ such that the good partition $\Omega_{k1}$ has smaller $c_{k1} < c_k$ and the bad partition $\Omega_{k2}$ has larger $\Delta_{k2} > \Delta_0$ and falls into $\kbad$, then we can improve the regret bound (Fig.~\ref{fig:theory}(b)-(c)). This coincides with the $\mathrm{Partition}$ function of LaMCTS very well: it samples a few points in $\Omega_k$, and trains a classifier to separate high $f$ from low $f$. On the other hand, if we partition a region $\Omega_k$ randomly, e.g., each $f(\vx)$ is assigned to either $\Omega_{k1}$ or $\Omega_{k2}$ at random, then statistically $F_{k1} = F_{k2} = F_k$ and $c_{k1} = c_{k2} = c_k$, which \emph{increases} the regret bound. Therefore, the partition needs to be \emph{informed} by data that have already been sampled within the region $\Omega_k$.

\textbf{Recursive region partitioning}. In Theorem~\ref{thm:regret}, we assume all regions $\{\Omega_k\}$ have fixed $c_k$ and $z_k$, so the bound breaks for large enough $T$ (as $\eta/T^3$ eventually becomes smaller than any fixed $z_k$). However, as LaMCTS conducts further internal partitioning within $\Omega_k$, its $c_k$ and $z_k$ keep shrinking with more samples $T$. If each split leads to slightly fewer bad $f$ (i.e., lighter ``tail''), with the ratio being $\gamma < 1$, then by the definition of CDF, $z_k$ is the probability mass of the tail and thus $z_k \sim \gamma^{- T / N_\mathrm{par}}$. This would yield $z_k \le \eta / T^3$ for all $T$, since $\gamma^{-T}$ decays faster than $1/T^3$ and Theorem~\ref{thm:regret} would hold for all $T$. See Appendix \ref{sec:empirical_zk} for empirical verification of decaying $z_k$.  

\subsection{Related Work and Limitations} \label{sec:related_works}
While related to Lipschitz bandits~\cite{magureanu14} and coarse-to-fine deterministic function optimization like DOO and SOO~\cite{munos2014bandits}, our analysis is fundamentally different. We have discussed how $f$-dependent region partitioning and a data-driven learning procedure affect the regret bound, which to our knowledge has not been previously addressed. See Appendix \ref{sec:theorem_remarks} for further remarks on Theorem~\ref{thm:regret}. 

There is more work to be done to fully understand how LaMCTS works. In particular, we did not analyze when to split a node (e.g. how many samples we need to collect before making a decision), or the effect of relearning the space partition. We also have not considered stochastic reward functions, where the maximum function value in the sub-region may no longer be the best metric of goodness. We leave these to future work. %\yuandong{Discuss about hierarchical structure -> top node can be treated as a regularization to make sure we don't sample crazily. No theory though.}

\section{\ours{} for Path Planning}\label{sec:method}
Based on our analysis, we propose \ours{}, which extends LaMCTS to path planning, a problem with temporal structure. \ours{} outperforms baseline path planning approaches in both continuous and discrete path planning problems. 
% Note that path planning has temporal structure, which is not present in the original LaMCTS paper. 
Here we represent trajectories as action sequences $\vx=(a_0, a_1, \ldots, a_{n-1})$ and treat them as high-dimensional vectors $\vx$ in the trajectory space $\Omega$. %We then search for the best trajectory in $\Omega$. 
% Compared to baselines, \ours{} can better escape local minima by exploration and can handle both continuous and discrete path planning problems.

% Given a starting state $s_0$, we would like to find a trajectory $\vx = (s_0, a_0, s_1, a_1, \ldots, s_n)$, represented in practice as $(a_0, a_1, \ldots, a_{n-1})$, that achieves the highest reward. Note that path planning has temporal structure, which is not present in the original LaMCTS paper. Here we treat each trajectory $\vx$ as a high-dimensional vector in the trajectory space $\Omega$ and search for the best trajectory in $\Omega$ efficiently. 

% While several representations of $\vx$ are possible, it is helpful to reduce the trajectory dimension to facilitate efficient search. Therefore, in practice we represent $\vx$ as $\vx=(a_0, a_1, \ldots, a_{n-1})$.

Thus, \ours{} searches over the space $\Omega$, recursively partitioning $\Omega$ into subregions based on trajectory reward, and sampling from subregions using CMA-ES~\cite{hansen2016cma} (which is faster than TuRBO~\cite{eriksson2019scalable} used in the original LaMCTS). We emphasize again that \ours{}'s region partitioning procedure is fully adaptive, in contrast to traditional MCTS approaches such as VOOT, which only partition the trajectory space based on one action at a time.

Additionally, we have made several improvements over the original LaMCTS, as detailed in Algorithm \ref{alg:plalam}. First, we use the maximal value $\max_{i\in \Omega_k} f(\vx_i)$ rather than the mean value $\frac{1}{n(\Omega_k)}\sum_{i\in \Omega_k} f(\vx_i)$ as the metric of goodness for each node $k$ (and its associated region $\Omega_k$). This is driven by Theorem~\ref{thm:regret}, which gives a regret bound based on maximum values. Intuitively, using the mean value would cause the algorithm to be slow to respond to newly discovered territory: it takes time for the mean metric to boost, and we may miss important leaves. We show the difference empirically in Sec. \ref{sec:analysis}. %as shown in Fig.~\ref{}. 

% Second, when drawing samples within a leaf, we use CMA-ES~\cite{hansen2016cma} rather than TuRBO~\cite{eriksson2019scalable} as a submodule, which substantially improves \ours{}'s speed compared to the original LaMCTS. 

Second, Theorem~\ref{thm:regret} suggests that a lower-dimensional (smaller $d$) and smoother (smaller $c_k$) representation leads to lower regret. Therefore, \ours{} employs a latent space as described below. 

\subsection{Latent Spaces For Partitioning and Sampling}

\ours{} leverages a latent space $\Phi_s$ for the \textit{partition} space, by passing $\Omega$ through some encoder $s$. That is, we disentangle the \textit{sampling} space $\Omega$ from which we sample new candidate trajectories, from the \textit{partition} space $\Phi_s$ on which we construct the search space partition. Critically, we do not need $s^{-1}$: we never decode from $\Phi_s$ back to $\Omega$. Thus $s$ can dramatically reduce the dimension of the partition space, which may improve regularization due to the small number of samples, without suffering large reconstruction loss. $s$ will be fixed rather than learned in this case. %, one may e.g., choose $s(\vx)$ to be a periodic snapshot of the state. 
%be merely a small fraction of the states or actions present in the full trajectory $\vx = (s_0, a_0, s_1, a_1, \ldots, s_n)$. 
Once the partition has been constructed on $\Phi_s$, and we select a leaf region to propose from, we sample new $\vx$ from $\Omega$ as before.\footnote{Specifically, we initialize the inner solver (CMA-ES in our experiments) using the pre-existing samples corresponding to the selected leaf region in $\Phi_s$, and then propose new samples using that initialization.}

In principle, the sampling space can itself be a latent space $\Phi_h$, with an encoder $h$ and decoder $h^{-1}$. That is, one runs the inner solver in $\Phi_h$ to propose samples before decoding back to $\Omega$. $h$ could be a principal component analysis (PCA)~\cite{wold1987principal}, a random network encoding~\cite{ulyanov2018deep}, or a reversible flow~\cite{dinh2016density}, depending on the environment's particular $\Omega$ and state/action structure. While some latent representations can be fixed by specifying the inductive bias 
(e.g., random network encoding), others can be learned from data, optimizing reconstruction loss
%others can be learned from previous experiences of solving similar motion planning problems, optimizing the following reconstruction objective:
% \begin{equation}
    $\min_h \eee{\vx}{w(\vx)\|h^{-1}(h(\vx)) - \vx\|^2}$, %\label{eq:latent-training}
% \end{equation}
where $w(\vx)$ is a weighting function emphasizing trajectories with high cumulative reward $f(\vx)$. In this case, $h$ and $h^{-1}$ may be fine-tuned using each new $(\vx, f(\vx))$ pair when \ours{} proposes and queries a new trajectory $\vx$, or they may be pre-trained using a set of unlabeled $\vx$ with $w(\vx) \equiv 1$. For consistency in our main experiments, we do not use a latent $\Phi_h$, although we observe that using this second latent space can yield a slight performance in some environments (Appendix \ref{sec:latent_space_ablations}).

\section{\ours{} on Synthetic Environments}\label{sec:synthetic_experiments}

We test \ours{} on a diverse set of environments to evaluate its performance in different settings. 

%\kevin{TODO high level summary of our envs later? Maybe this Experimental Setup header can be deleted, idk} \tianjun{I think we could leave this part ot sec.6}

\textbf{Baselines}. We compare \ours{} to several baselines. \textbf{LaMCTS} is the original LaMCTS algorithm using CMA-ES as an inner solver, like \ours{}. \textbf{Random Shooting (RS)}~\cite{richards2005robust} samples random trajectories and returns the best one. \textbf{Cross-Entropy Methods (CEM)}~\cite{botev2013cross} use the top-$k$ samples to fit a local model to guide future sampling. A related approach, \textbf{Covariance matrix adaptation evolution strategy (CMA-ES)}~\cite{hansen2016cma}, tracks additional variables for improved local model fitting. \textbf{Voronoi optimistic optimization applied to trees (VOOT)}~\cite{voot} is a ``traditional'' MCTS method for continuous action spaces that builds a tree on actions at each timestep. \textbf{iLQR}~\cite{li2004iterative} is a seminal gradient-based local optimization approach used extensively in controls. Finally, \textbf{proximal policy optimization (PPO)}~\cite{schulman2017proximal} is a standard reinforcement learning algorithm. 

\iffalse
While \ours{} is more complex than the baselines, we emphasize that 
\fi
\ours{} does not require substantially more tuning effort than CEM or CMA-ES, the best-performing among our baselines experimentally. 
%In fact, \ours{} does not even require tuning $\sigma$ in the CMA-ES submodule we use to propose samples within leaves, and we do not tune this parameter in our experiments; rather, we use the standard deviation within the leaf to initialize $\sigma$ for CMA-ES. 
The only additional hyperparameter tuned in \ours{} is the $C_p$ controlling exploration when selecting regions to sample from, which is dependent on the scale of the reward function. However, our $C_p$ only varies by a factor of up to 10 across our diverse environments, and performance is not overly sensitive to small changes (Appendix \ref{sec:hyperparameter_sensitivity}).

We use MiniWorld~\cite{gym_miniworld} for continuous path planning and MiniGrid~\cite{gym_minigrid} for discrete. 

\subsection{MiniWorld}

\begin{figure*}[htbp]
    \centering
    \begin{subfigure}[t]{0.32\textwidth}
        \centering
        \includegraphics[height=1.1in]{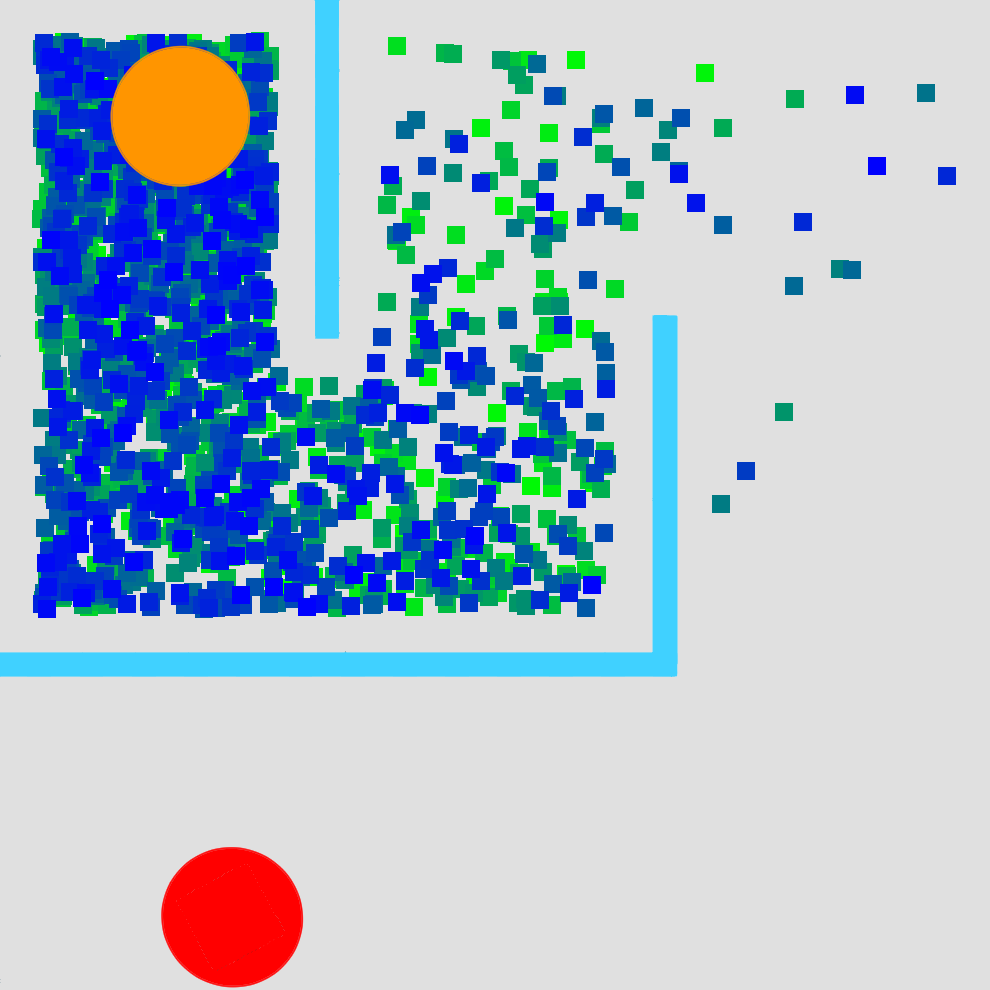}
        \vspace{-5pt}
        \caption{RS}
    \end{subfigure}%
    \hfill
    \begin{subfigure}[t]{0.32\textwidth}
        \centering
        \includegraphics[height=1.1in]{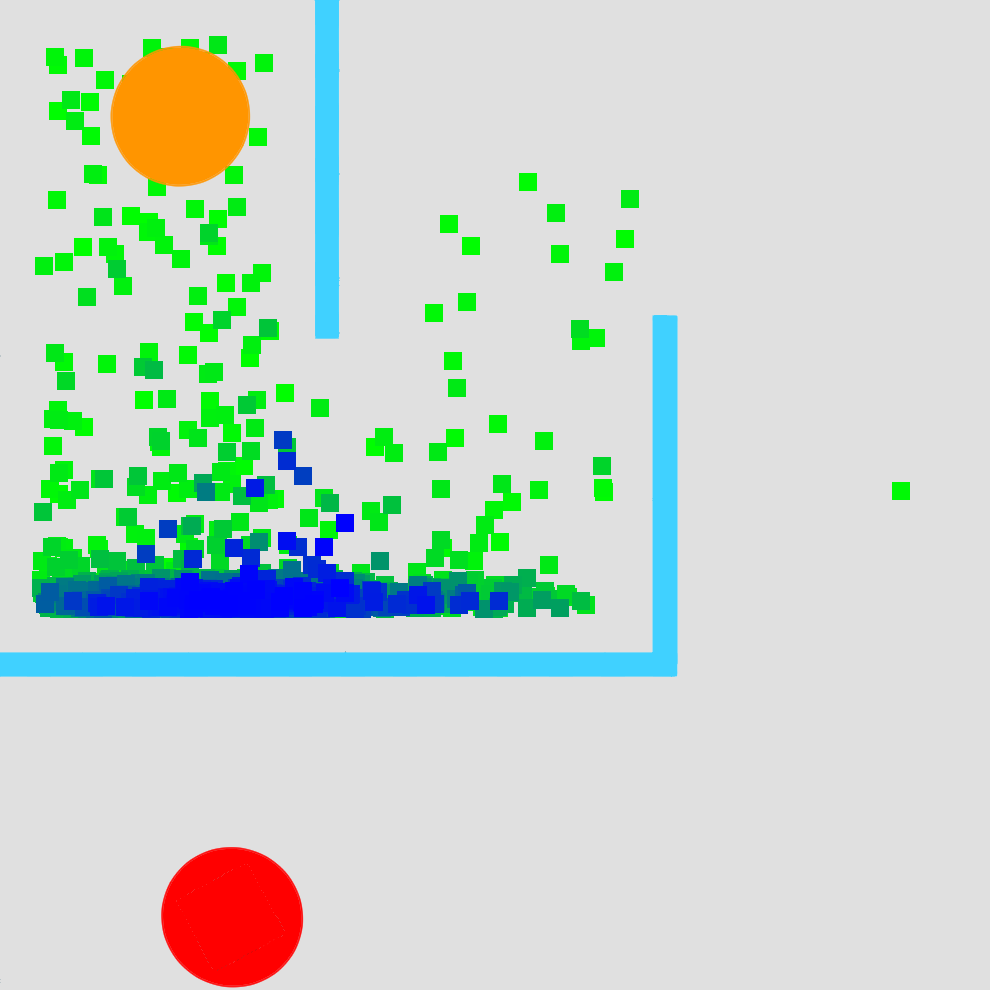}
        \vspace{-5pt}
        \caption{CEM}
    \end{subfigure}
    \hfill
    \begin{subfigure}[t]{0.32\textwidth}
        \centering
        \includegraphics[height=1.1in]{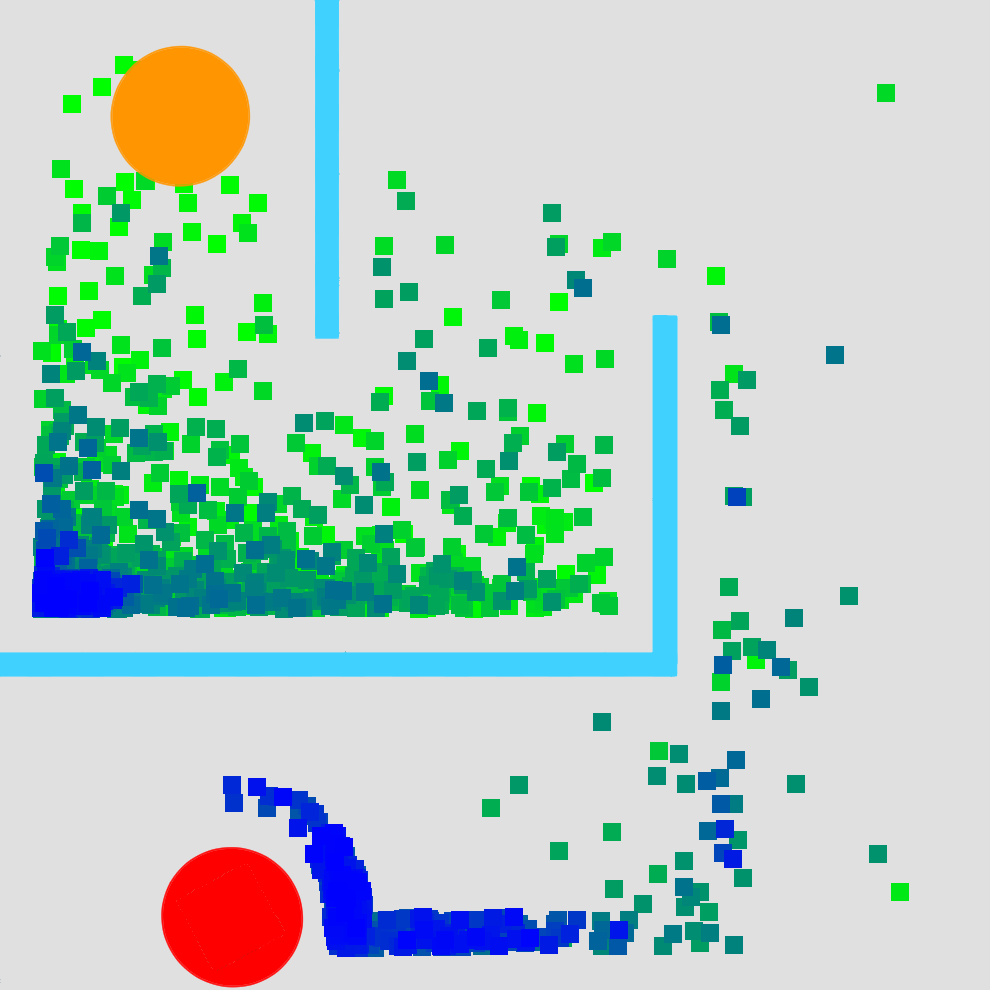}
        \vspace{-5pt}
        \caption{\ours{}}
    \end{subfigure}
    \vspace{-5pt}
    \caption{\small MazeS3 environment. \textbf{Start:} Orange circle. \textbf{Goal:} Red circle. Dots indicate final agent positions of 2,000 proposed trajectories (green: first iteration, blue: last iteration). CEM gets stuck in a local optimum of reward (shown as concentration of blue dots), while \ours{} succeeds in reaching the goal. %\yuandong{We can put three/four figures (e.g., RS, CEM, CMA-ES, \ours{}) in a row, which can make the figure shorter. Also max number of iterations?}
    }
    \label{fig:miniworld_example}
\end{figure*}

We consider the following 2D navigation tasks in \emph{MiniWorld}.
\textbf{MazeS3}: Agent navigates in a 3 by 3 maze to a goal. Greedy path planning gets stuck in local optima (Figure \ref{fig:miniworld_example}).
\textbf{FourRooms}: Agent navigates from one room in a 2 by 2 configuration to a goal in the diagonally opposite room. Greedy path planning gets stuck in a corner. 
\textbf{SelectObj}: Open space with two goals. Large final reward when reaching the farther goal, while a distance-based reward misleadingly points to the closer goal. 
% In all environments, 
For full environment specifics, see Appendix \ref{sec:miniworld_details}. 

We modify the original setup to use a continuous action space ($\Delta x$ and $\Delta y$), and provide a sparse reward (proximity to goal, with an additional bonus for reaching the goal) at end-of-episode. We use a high-dimensional top-down image view as the state. We featurize this image using a randomly initialized convolutional neural network, a reasonable feature extractor as shown in \cite{ulyanov2018deep}. \ours{} uses periodic snapshots of the featurized state as the partition space $\Phi_s$. That is, we collect all the observed states over the course of the full trajectory, and then form the latent space by concatenating every $n^{th}$ state (here $n=20$), while discarding the rest to reduce overall dimensionality. Success is defined using a binary indicator for reaching the goal (far goal for SelectObj). 

\iffalse
While \ours{} shows few gains over the baselines in the relatively smooth Mujoco tasks, many tasks are not so smooth, and contain local optima which methods like CEM, CMAES, and especially iLQR may struggle to escape. We explore this possibility in three challenging environments built in MiniWorld \cite{gym_miniworld}.
\fi

% \setlength{\textfloatsep}{5pt plus 1.0pt minus 2.0pt}
\begin{figure*}[t!]
    \centering
    \begin{subfigure}[t]{0.32\textwidth}
        \centering
        \includegraphics[height=1.4in]{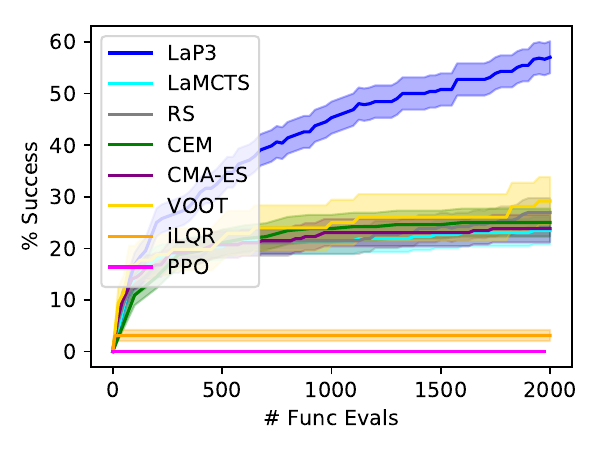}
        \vspace{-20pt}
        \caption{MazeS3}
    \end{subfigure}
    \hfill
    \begin{subfigure}[t]{0.32\textwidth}
        \centering
        \includegraphics[height=1.4in]{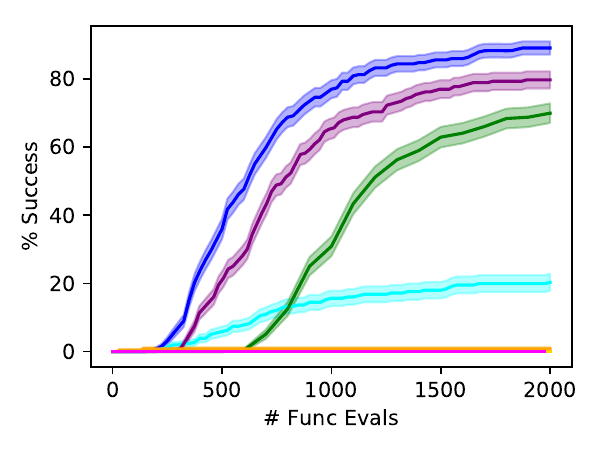}
        \vspace{-20pt}
        \caption{FourRooms}
    \end{subfigure}%
    \hfill
    \begin{subfigure}[t]{0.32\textwidth}
        \centering
        \includegraphics[height=1.4in]{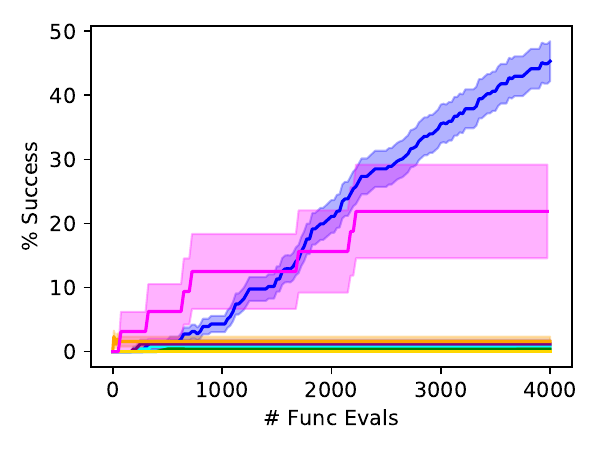}
        \vspace{-20pt}
        \caption{SelectObj}
    \end{subfigure}
    \vspace{-5pt}
    \caption{\small Mean, and standard deviation of mean (256 trials; fewer for VOOT and PPO due to speed), of success rate across MiniWorld tasks. \ours{} significantly outperforms all baselines on all three tasks.}
    % \vspace{-5pt}
    \label{fig:miniworld_graph}
\end{figure*}

\begin{figure*}[t!]
    
    \centering
    \hfill
    \begin{subfigure}[t]{0.32\textwidth}
        \centering
        \includegraphics[height=1.4in]{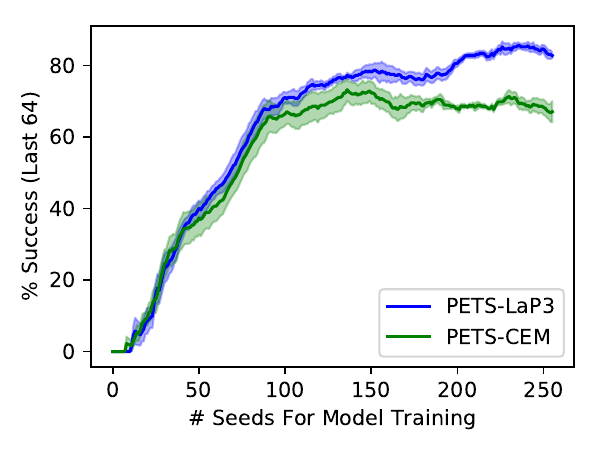}
        \vspace{-20pt}
        \caption{PETS FourRooms}
    \end{subfigure}%
    \hfill
    \begin{subfigure}[t]{0.32\textwidth}
        \centering
        \includegraphics[height=1.4in]{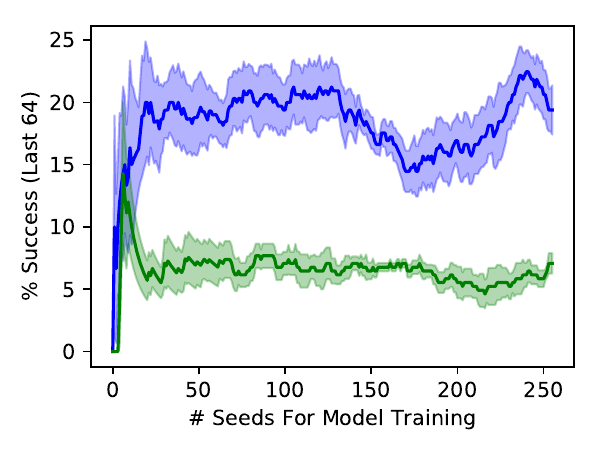}
        \vspace{-20pt}
        \caption{PETS SelectObj}
    \end{subfigure}
    \hfill
    \vspace{-5pt}
    \caption{\small \ours{} in PETS compared to original PETS planners on MiniWorld environments, using PETS-learned world models. Sliding length-64 window of success percentage against number of training seeds for world model. \ours{} significantly outperforms all baselines on both tasks. 
    %(MazeS3 omitted due to environment walls varying with seed, making it difficult to learn a world model.) 
    %\kevin{Should I show the first 64 steps in these plots too? Also I might get better results on SelectObj-- I think I ran our method with bad hyperparams, will just update later}\yuandong{64 random init? We can show them}\kevin{updated--the issue is that the sliding window is length-64 but early on you don't have that many trials so it's noisy as you can see in SelectObj, not sure if this is confusing}
    }
    \label{fig:miniworld_pets_graph}
\end{figure*}

\textbf{Results}. \ours{} substantially outperforms all baselines on all three tasks, despite heavily tuning the  baselines' hyperparameters (Appendix \ref{sec:hyperparameter_tuning}), showing that \ours{} works for challenging tasks containing suboptimal local maxima.
% Note that we  but to no avail. 
In MazeS3, \ours{} succeeds but CEM gets stuck (Figure \ref{fig:miniworld_example}). VOOT, which builds an MCTS tree on actions at each timestep, struggles on all environments; \ours{} can be viewed as an extension of MCTS that performs better on such long-horizon tasks. PPO also performs poorly, perhaps due to the sparse reward given only at the end of an episode, and the relatively small (for RL) number of episodes. In the most difficult SelectObj task, \ours{} solves nearly half of environment seeds within 4,000 queries of the oracle, whereas most baselines---including the original LaMCTS---quickly reach the near goal but struggle to escape this local optimum. 
%while all baselines achieve much lower success rate. CEM, CMAES, VOOT, and iLQR quickly find paths reaching the near goal, but struggle to escape this local optimum. 

We also evaluate \ours{} when combined with a model-based approach, PETS~\cite{chua2018deep}, on FourRooms and SelectObj (omitting MazeS3 because the changing maze walls for each seed make it difficult to learn a world model). Following PETS' setting and due to difficulty in learning image-based world models \cite{hafner2019dream,hafner2020mastering}, we use 2D agent position as the state. As shown in Fig.~\ref{fig:miniworld_pets_graph}, \ours{} substantially outperforms the authors' original CEM implementation in the PETS framework, demonstrating that it is not reliant on access to the oracle model but can work with learned models as well.

\iffalse
We compare with PET's original CEM implementation and report success rate after 256 trials, noting that the learned PETS transition model improves over time.  
\fi

\subsection{MiniGrid}
\emph{MiniGrid}~\cite{gym_minigrid} is a popular sparse-reward symbolic environment for benchmarking RL algorithms. It contains tasks with discrete states and actions such as \textbf{DoorKey~(DK)}: pick up a key and open the door connecting two rooms; \textbf{MultiRoom~(MR)}: traverse several rooms by opening doors; and \textbf{KeyCorridor~(KC)}, a combination of MR and DK: some doors are locked and require a key. As in MiniWorld, we add proximity to the goal to the final sparse reward.

In discrete action spaces, \ours{} optimizes the vector of all action probabilities over all timesteps, and takes the highest-probability action at each step. 
% Since the environment is procedurally-generated, we run 256 seeds and report the mean and variance.
As in MiniWorld, we use periodic state snapshots featurized by a randomly initialized CNN as the partition space $\Phi_s$. %The sampling space $\Phi_h$ is given by a reversible flow model \cite{dinh2016density} trained on \ours{}'s generated trajectories. 
We compare \ours{} to the same baselines as in MiniWorld, except VOOT and iLQR which are designed for continuous tasks.

\textbf{Results}. \ours{} is equal to or better than baselines on all six tasks (Table~\ref{tab:minigrid_results}). 
%In simple tasks like DK-6 and KC-S3R3 where CMA-ES is even doing worse than RS, \ours{} with latent representation can still improve over the performance of baselines. 
Especially in the hardest tasks with the most rooms (MR-N4S5, MR-N6), \ours{} improves substantially over baselines. 

% \brandon{Minigrid: How much better is this? Since the variance is quite large in these experiments so the means are well within a standard deviation of each other?}

\begin{table*}[!htb]
\caption{Results for \ours{} in MiniGrid. \ours{} is equal or better on all tasks (higher is better).}
\footnotesize
\setlength\tabcolsep{3.5pt}
\label{tab:minigrid_results}
\centering
\begin{tabular}{p{2cm}p{1.5cm}p{1.5cm}p{1.5cm}p{1.5cm}p{2cm}p{2cm}}
% {lcccccccccccc}
\toprule
& \textbf{DK-6} & \textbf{DK-8} & \textbf{KC-S3R3} & \textbf{KC-S3R4} & \textbf{MR-N4S5} & \textbf{MR-N6}\\ 
\midrule
LaMCTS & \textbf{0.96}$\pm$\textbf{0.02} & 0.09 $\pm$ 0.17 & -2.63$\pm$0.09 & -4.43$\pm$0.13 & -14.71$\pm$0.87 & -118.70$\pm$4.68 \\
% 0.96 +/- 0.02	0.09 +/- 0.17	-2.63 +/- 0.09	-4.43 +/- 0.13	-14.71 +/- 0.87	-118.70 +/- 4.68
RS & \textbf{0.97}$\pm$\textbf{0.01} & \textbf{0.34$\pm$0.13} & -2.38$\pm$0.09 & \textbf{-4.27$\pm$0.12} & -18.16$\pm$0.80 & -119.39$\pm$4.64 \\
CEM & 0.03$\pm$0.12 & -3.34$\pm$0.34 & -3.40$\pm$0.08 & -4.93$\pm$0.13 & -22.88$\pm$1.00 & -131.32$\pm$5.24\\
CMA-ES & 0.93$\pm$0.03 & 0.23$\pm$0.14 & -2.46$\pm$0.09 & -4.44$\pm$0.12 & -14.31$\pm$0.78 & \textbf{-117.50$\pm$4.61} \\
\ours{} & \textbf{0.95$\pm$0.03} & \textbf{0.46$\pm$0.13} & \textbf{-2.27$\pm$0.09} & \textbf{-4.37$\pm$0.13} & \textbf{-11.68$\pm$0.75} & \textbf{-113.53$\pm$4.49}  \\
%0.95 +/- 0.03	0.46 +/- 0.13	-2.27 +/- 0.09	-4.37 +/- 0.13	-11.68 +/- 0.75	-113.53 +/- 4.49
% \ours{} & \textbf{0.97}$\pm$\textbf{0.02} & \textbf{0.48}$\pm$\textbf{0.11} & \textbf{-2.19}$\pm$\textbf{0.15} & \textbf{-4.22}$\pm$\textbf{0.13} & \textbf{-10.68}$\pm$\textbf{0.68} & \textbf{-112.72}$\pm$\textbf{4.46} \\
% ours state & & & & -2.27$\pm$1.48 & -4.37$\pm$2.08\\
% ours Slatent & & & & -2.19$\pm$2.47 & -4.22$\pm$2.10 \\
\bottomrule 
\end{tabular}
\end{table*} 

\subsection{Analysis}\label{sec:analysis}
We run several ablations on \ours{} in MiniWorld to justify our methodological choices. See Appendix \ref{sec:detailed_analysis} for further analysis on hyperparameter sensitivity, UCB metric, and latent spaces.

\iffalse
First, in Table \ref{tab:leaf_selection}, w
\fi

\textbf{Region Selection in \ours}. We consider four alternative region selection methods. (1) \ours{}-$\mathrm{mean}$: using mean function value rather than max for UCB, as in LaMCTS \cite{wang2020learning}; \textbf{(2)} \ours{}-$\mathrm{nolatent}$: not using a latent space for partitioning; \textbf{(3)} \ours{}-$\mathrm{notree}$: directly selecting the leaf with the highest UCB score; and \textbf{(4)} \ours{}-$\mathrm{noUCB}$: only using node value rather than UCB. \ours{} greatly outperforms all variations in MiniWorld, justifying our design. 

\newfloatcommand{capbtabbox}{table}[][\FBwidth]

\begin{figure}[htbp]
\begin{floatrow}
\CenterFloatBoxes
\floatsetup{heightadjust=object,captionskip=0pt}
% \begin{figure}[t]{0.4\textwidth}
\ffigbox[\FBwidth][\FBheight][t]{
% \rule{3cm}{3cm}%

        % \centering
        \includegraphics[height=1.2in]{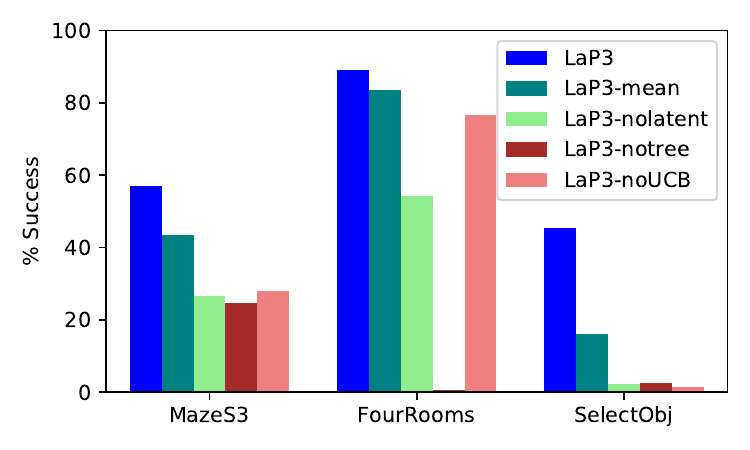}

        % \vspace{10pt}
        % \vspace{-20pt}

}
{
\caption{\small MiniWorld success percentages with different region selection methods.}
\label{fig:region_select}
}
% \end{figure}
% \vspace{-100pt}

\capbtabbox{%
  \begin{tabular}{lccc}
            \toprule
& \textbf{MazeS3} & \textbf{FourRooms}  & \textbf{SelectObj }\\ 
\midrule  
$L_k$ & 87.5 & 100.0 & 100.0\\
$c_k$ & 81.3 & 93.8 & 100.0\\
\bottomrule 
% \vspace{-10pt}
\caption{\small Percentage out of 32 environment seeds on MiniWorld environments where \ours{} yields a better estimated Lipschitz and $c_k$ compared to random partitioning on the same nodes. }%
\label{tab:miniworld_lipschitz}

        \end{tabular}
        % \vspace{50pt}

}{%
%   \vspace{40pt}
  
}
\end{floatrow}
\end{figure}

\textbf{Data-driven space partition in \ours{} vs. random partitioning}.  We examine $c_k$ in Def.~\ref{def:diluted} and Lipschitz constant $L_k$ in Corollary~\ref{corollary:lipschitz} to verify the theory. We conduct a preliminary analysis on \ours{}'s tree after the full 2,000 queries (4,000 for SelectObj). At each intermediate node, we estimate $L_k$ and $c_k$ of its children from the \ours{} partition, against a random partition that divides the node's samples with the same ratio (see Appendix \ref{sec:lipschitz_estimation} for estimation details). We then average the values for both \ours{} and random partitions over all nodes in the tree. We find that \ours{} does yield lower average $L_k$ and $c_k$ (Table \ref{tab:miniworld_lipschitz}), indicating that our data-driven space partition is effective.

\section{\ours{} on Real-World Applications}\label{sec:real_world_experiments}
% We test \ours{} on two real-world tasks: compiler phase ordering~\citep{ansel2014opentuner} and molecular design~\citep{vamathevan2019applications}. 

% Miniworld
% \begin{enumerate}
%     \item Modified continuous 4 rooms: LaMCTS vs CEM baselines (also CMAES/other simple planning approaches if needed?), oracle model [basically done, and works]. 
%     \item Modified continuous 4 rooms: LaMCTS vs CEM baselines (also CMAES/other simple planning approaches if needed?), PETS [basically done, running numbers].
%     \item Modified continuous 4 rooms: Ablation with parameter space LaMCTS. [yangk queue 1]
% \end{enumerate}

\iffalse
A program can get significantly better performance when applying the right optimization before feeding it to the compiler~\cite{haj2020autophase}. 
The optimization may consist of several phases, with each phase applying 
\fi

\subsection{Compiler Phase Ordering}

\begin{figure}[tb!]
    \centering
    \includegraphics[width=\textwidth]{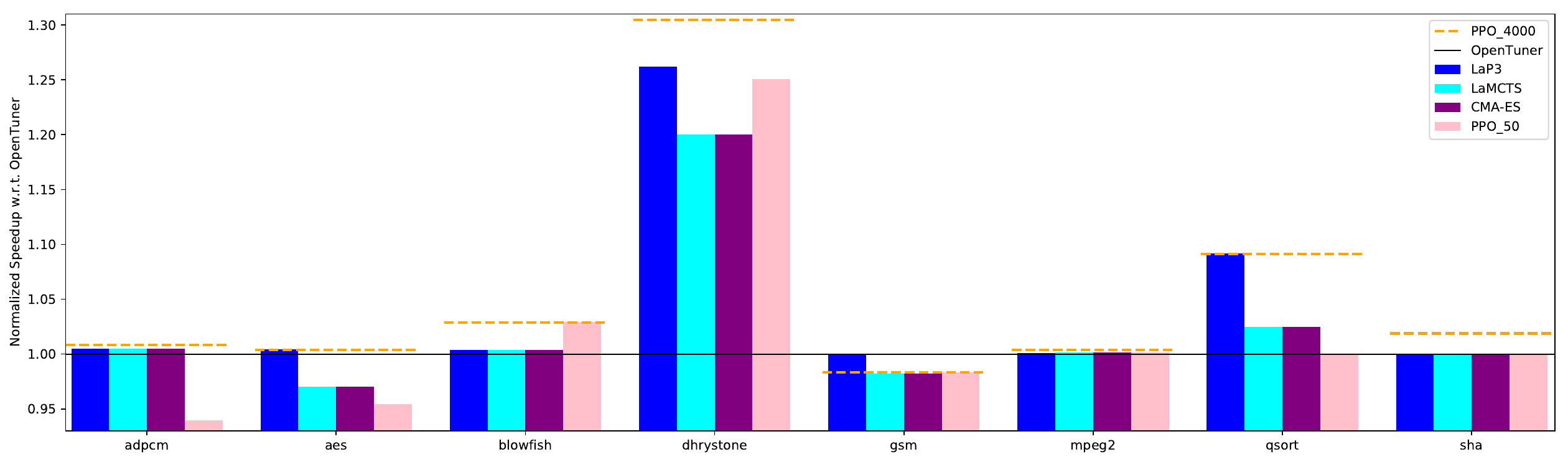}
    \vspace{-0.2in}
    \caption{\small Compiler phase ordering results, in terms of normalized execution cycles with respect to OpenTuner~\cite{ansel2014opentuner}, a widely used method for program autotuning. \ours{} is consistently equal or better compared to baselines. We omitted the $\mathrm{matmul}$ task since it doesn't fit the scale with its 245\% speedup over OpenTuner.
    %In $\mathrm{gsm}$, \ours{} gets a better solution using only 50 samples even comparing with PPO using 4000 samples.
    } 
    \label{fig:autophase}
\end{figure}

Compiler optimization applies a series of program transformations from a set of predefined optimizations (e.g., \emph{loop invariant code motion}, \emph{function inlining}~\cite{muchnick1997advanced}) to improve code performance. Since these optimizations are not commutative, the order in which they are applied is extremely important. This problem, known as \textit{phase ordering}, is a core challenge in the compiler community.
Current solutions to this NP-hard problem rely heavily on heuristics: groups of optimizations are often packed into "optimization levels" (such as -O3 or -O0) hand-picked by developers~\cite{pan2006fast, triantafyllis2003compiler}.

We apply \ours{} to the standard CHStone benchmarks~\cite{hara2008chstone}, and use periodic snapshots of states as $\Phi_s$ and the identity as $\Phi_h$. 
\iffalse
The action space consists of 46 different program transformations, and a trajectory consists of 45 transformations (quite short, considering many transformations have no effect unless applied in a specific order). The reward is the difference between the original and final number of execution cycles. Since the environment is deterministic, we only run 1 trial for each method. Thus far we have followed the setup in \citep{haj2020autophase}; however, unlike \citep{haj2020autophase}, we allow a budget of only 50 trajectory queries.
\fi
See Appendix~\ref{sec:autophase_details} for full environment details.

% We apply \ours{} to CHStone~\cite{hara2008chstone} benchmarks (Figure \ref{fig:autophase}).\footnote{We exclude $\mathrm{matmul}$ in Figure ~\ref{fig:autophase} since the performance of OpenTuner, which we normalize to, is poor. We refer the readers to Appendix.~\ref{} for the complete results.} 

\textbf{Results}. \ours{} is 31\% faster on average compared to OpenTuner, and 39\% compared to -O3 (not shown in figure). 
% We also achieve an average of $19\%$ speedup over the widely used OpenTuner method~\cite{ansel2014opentuner}. 
Compared to a stronger PPO baseline using 50 samples (PPO\_50) and to CMA-ES, we achieve up to $10\%$ and $7\%$ speedup respectively.
% Comparing \ours{} to a stronger PPO baseline using 50 samples (PPO\_50), we still achieve up to 10\% speedup. \ours{} also improves 5\% to 7\% over CMAES. 
Finally, compared to final PPO results at convergence after 4000 samples (PPO\_4000) as an oracle, \ours{} does similarly on most tasks, despite being much more sample efficient (only 50 samples). Full results in Appendix~\ref{sec:numerical_tables}.
%\ours{} yields comparable and sometimes even better performance (e.g., $\mathrm{gsm}$ task).  
% \chris{Further details would be useful here, like exact train/test sets. Perhaps reference supplementary material.}

\iffalse
for success on this task.
However, without the use of a latent space in the molecular domain, it is difficult even to sample well-formed SMILES strings, let alone molecules with desired properties. Thus, 
\fi

\iffalse
is used.
Commonly in drug development, one may have access to a large bank of generically drug-like molecules from prior research, but without labels for the particular property of interest --- for example, effectiveness in inhibiting a recently discovered bacteria or virus. Therefore, we assume access to 

These molecules are used to pretrain the encoder and decoder to a latent space, 
the task more challenging by reducing the likelihood of sampling molecules with high property scores from the latent space by chance. 

We thereafter assume no access to property scores, as in a practical scenario, thus precluding e.g., offline reinforcement learning methods \cite{levine2020offline}. 
\fi

\subsection{Molecular Design}

\FloatBarrier

\begin{figure*}[htbp]
    \centering
    \begin{subfigure}[t]{0.24\textwidth}
        \centering
        \includegraphics[height=1.05in]{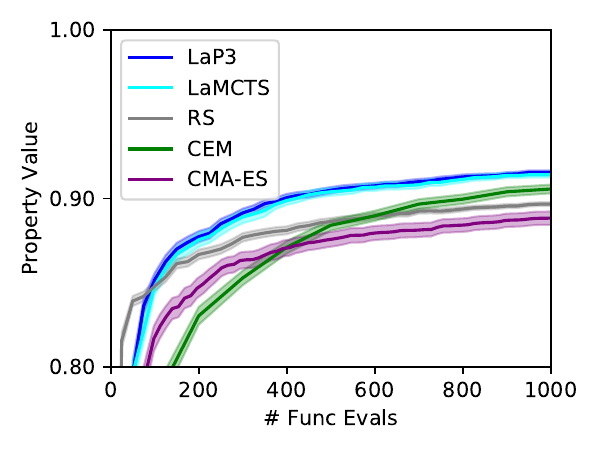}
        \vspace{-20pt}
        \caption{QED}
    \end{subfigure}%
    \hfill
    \begin{subfigure}[t]{0.24\textwidth}
        \centering
        \includegraphics[height=1.05in]{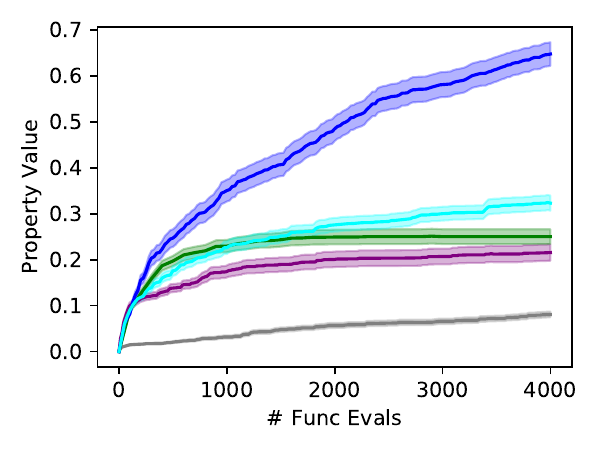}
        \vspace{-20pt}
        \caption{DRD2}
    \end{subfigure}
    \hfill
    \begin{subfigure}[t]{0.24\textwidth}
        \centering
        \includegraphics[height=1.05in]{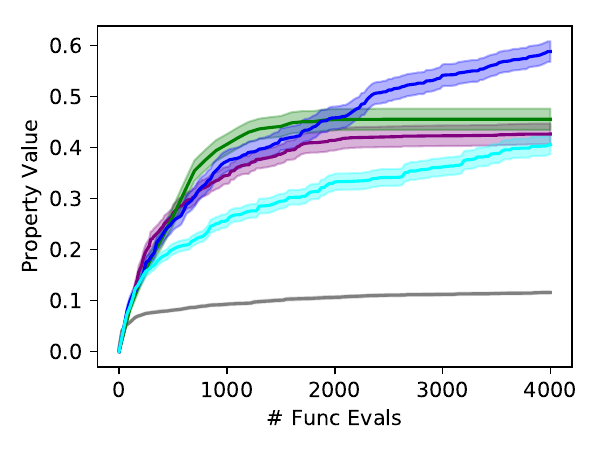}
        \vspace{-20pt}
        \caption{HIV}
    \end{subfigure}
    \hfill
    \begin{subfigure}[t]{0.24\textwidth}
        \centering
        \includegraphics[height=1.05in]{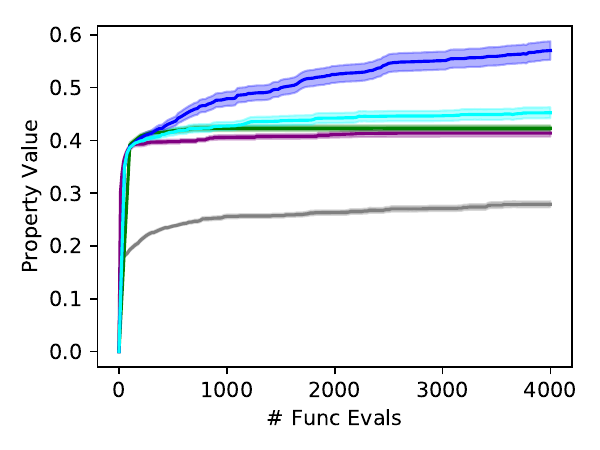}
        \vspace{-20pt}
        \caption{SARS}
    \end{subfigure}
    \vspace{-5pt}
    \caption{\small Mean and standard deviation (128 trials), of max property value discovered in molecular design tasks. \ours{} significantly outperforms all baselines on all properties. %\yuandong{We can choose to show only one: Table 6 or this figure}.
    }
    \label{fig:molecule_graph}
\end{figure*}

Finally, we evaluate \ours{} on molecular design. Given an oracle for a desired molecular property, the goal is to generate molecules with high property score after the fewest trials. This is critical to pharmaceutical drug development \cite{vamathevan2019applications}, as property evaluations require expensive wet-lab assays. 

Similar to \cite{jin2018junction}, we fix a query budget and optimize several properties: \textbf{QED}: a synthetic measure of drug-likeness, relatively simpler to optimize; \textbf{DRD2}: a measure of binding affinity to a human dopamine receptor; \textbf{HIV}, the probability of inhibition potential for HIV; and \textbf{SARS}: the same probability for a variant of the SARS virus, related to the SARS-CoV-2 virus responsible for COVID-19. All four properties have a range of $[0, 1]$; higher is better. For DRD2, HIV, and SARS, we evaluate using computational predictors from \cite{olivecrona2017molecular} (DRD2) and \cite{yang2019analyzing} (HIV, SARS) in lieu of wet-lab assays. 

To run \ours{} on molecular design, we view the molecular string representation (SMILES string \cite{weininger1988smiles}) as the action sequence, similar to how many generative models generate molecules autoregressively~\cite{gomez2018automatic,kusner2017grammar,dai2018syntax,yang2020improving}. Following the state-of-the-art HierG2G model from \cite{jin2020hierarchical}, we learn a latent representation from a subset of ChEMBL~\cite{chembl}, a dataset of 1.8 million drug-like molecules, \emph{without} using any of its property labels (e.g., effectiveness in binding to a particular receptor). During this unsupervised training, we only use the 500k molecules with the lowest property scores to ensure a good molecule is discovered by search rather than a simple retrieval from the dataset. Our setting differs from many existing methods for molecular design, which assume a large preexisting set of molecules with the desired property for training the generator \cite{olivecrona2017molecular,jin2018learning,you2018graph,yang2020improving}.

On this task only, the latent space is trained on additional unlabeled data, and is used as both the partition space $\Phi_s$ and sampling space $\Phi_h$ for \ours{}. All baselines operate in the same space for fair comparison. Otherwise, all methods struggle to generate well-formed molecules of reasonable length. 

\iffalse
However, the iLQR baseline cannot easily be run on this task despite the continuous latent space, due to the discretization involved in decoding to actual molecules. 
\fi
\iffalse
While the standard deviation is relatively large for individual trials, a significance analysis indicates that \ours{} is better than all baselines in each task with $p < 0.001$ in all four tasks. 
\fi

\textbf{Results}.
%\brandon{Molecule: are there common molecular design algos used and if so do you have numbers for them? I think it'd just help bolster the case for PlaLaM and add context to the numbers/if PlaLaM is already SOTA compared to those that's a big win} (added discussion of how our setting differs from that of many prior methods)
Figure \ref{fig:molecule_graph} shows the highest property score discovered by each method for each property. 
%See Appendix \ref{sec:numerical_tables} for the final average numerical results. 
The absolute difference is small in the relatively simple synthetic QED task. However, \ours{} outperforms all baselines by a much greater margin---up to 0.4 in DRD2---in the more challenging and realistic DRD2, HIV, and SARS tasks, where CEM and CMA-ES quickly plateau but \ours{} continues to improve with more function evaluations. %\ours{} outperforms all baselines by up to $159\%$ in absolute reward (DRD2). 

\section{Conclusion}
We propose \ours{}, a novel meta-algorithm for path planning that learns to partition the search space so that subsequent sampling focuses more on promising regions. We provide a formal regret analysis of region partitioning, motivating improvements that yield large empirical gains. 
% Specifically, we use a latent representation of the search space, and use the maximum sample value rather than the mean for region selection.
\ours{} particularly excels in environments with many difficult-to-escape local optima, substantially outperforming strong baselines on 2D navigation tasks as well as real-world compiler optimization and molecular design. 
%In the future, we hope to empirically verify \ours{} on additional path planning tasks and also black-box optimization more generally. \benoit{Cut this last sentence ? }

\iffalse
Ethics section (?)
Acknowledgements
\fi

\begin{ack}
We thank the members of the Berkeley NLP group as well as our four anonymous reviewers for their helpful feedback. This work was supported by Berkeley AI Research, and the NSF through a fellowship to the first author. 
\end{ack}

\bibliography{references}
\bibliographystyle{plain}

\clearpage

\section*{Checklist}

\begin{enumerate}

\item For all authors...
\begin{enumerate}
  \item Do the main claims made in the abstract and introduction accurately reflect the paper's contributions and scope?
    \answerYes{We claim to provide a theoretical explanation of region partitioning and empirical gains over baselines, which are presented in Sec. \ref{sec:theory} and Secs. \ref{sec:synthetic_experiments},\ref{sec:real_world_experiments} respectively. }
  \item Did you describe the limitations of your work?
    \answerYes{We have discussed limitations of our preliminary theory in Sec. \ref{sec:implications}. Our latent spaces also inherently depend on the details of the environments, as described in each individual experiment section. While \ours{} could be easily modified for black-box optimization in principle, we have made clear that we empirically verify only on path planning.}
  \item Did you discuss any potential negative societal impacts of your work?
    \answerNo{We do not foresee any obvious negative societal impacts from our work, which contributes a general-purpose path planning algorithm.}
  \item Have you read the ethics review guidelines and ensured that your paper conforms to them?
    \answerYes{}
\end{enumerate}

\item If you are including theoretical results...
\begin{enumerate}
  \item Did you state the full set of assumptions of all theoretical results?
    \answerYes{In lemma/theorem statements in Sec. \ref{sec:theory}.}
	\item Did you include complete proofs of all theoretical results?
    \answerYes{All proofs are in Appendix \ref{sec:proofs}.}
\end{enumerate}

\item If you ran experiments...
\begin{enumerate}
  \item Did you include the code, data, and instructions needed to reproduce the main experimental results (either in the supplemental material or as a URL)?
    \answerYes{We upload code in the supplementary material.}
  \item Did you specify all the training details (e.g., data splits, hyperparameters, how they were chosen)?
    \answerYes{We discuss all hyperparameter tuning details in Appendix \ref{sec:hyperparameter_tuning}.}
	\item Did you report error bars (e.g., with respect to the random seed after running experiments multiple times)?
    \answerYes{Included in all experiments in Secs. \ref{sec:synthetic_experiments},\ref{sec:real_world_experiments}.}
	\item Did you include the total amount of compute and the type of resources used (e.g., type of GPUs, internal cluster, or cloud provider)?
    \answerNo{}
\end{enumerate}

\item If you are using existing assets (e.g., code, data, models) or curating/releasing new assets...
\begin{enumerate}
  \item If your work uses existing assets, did you cite the creators?
    \answerYes{}
  \item Did you mention the license of the assets?
    \answerNo{}
  \item Did you include any new assets either in the supplemental material or as a URL?
    \answerNA{Just code, which is in supplemental material.}
  \item Did you discuss whether and how consent was obtained from people whose data you're using/curating?
    \answerNo{We use publicly available datasets/tasks.}
  \item Did you discuss whether the data you are using/curating contains personally identifiable information or offensive content?
    \answerNo{We don't use data of this sort.}
\end{enumerate}

\item If you used crowdsourcing or conducted research with human subjects...
\begin{enumerate}
  \item Did you include the full text of instructions given to participants and screenshots, if applicable?
    \answerNA{}
  \item Did you describe any potential participant risks, with links to Institutional Review Board (IRB) approvals, if applicable?
    \answerNA{}
  \item Did you include the estimated hourly wage paid to participants and the total amount spent on participant compensation?
    \answerNA{}
\end{enumerate}

\end{enumerate}

\clearpage

\appendix

\def\circle#1{{\small \textcircled{\raisebox{-0.9pt}{#1}}}}

\section{LaMCTS Partition Function}\label{sec:alg_partition}

Algorithm~\ref{alg:partition_func} details the pseudocode for the partition function used in LaMCTS, which we use in \ours{} as well. 

\begin{algorithm}[htbp]
    \caption{Partition Function}
    \label{alg:partition_func}
    \begin{algorithmic}[1]
    \State {\bfseries Input:} Input Space $\Omega$, Samples $\cS_t$, Node partition threshold $N_{\thres}$, Partitioning Latent Model $s(\vx)$
    % \State Set $\Omega \gets \Omega \cup \cS_t$
    \State Set $\cV_0 = \{ \Omega \}$
    \State Set $\cV_{queue} = \{ \Omega \}$
    \While{$\cV_{queue} \neq \emptyset$}
    \State $\Omega_p \gets \cV_{queue}.pop(0)$
    \If{$n(\Omega_p) \ge N_{\thres}$}
    % \If{$s(\Omega_p)$ $k$-means splitable}
    \State $S_{good}, S_{bad} \gets$ samples from $S_t$ corresponding to indices of $k$-means$(\textcolor{darkgreen}{s(\Omega_p \cap S_t)})$
    \State Fit SVM on $S_{good}, S_{bad}$
    \State Use SVM to split $\Omega_p$ into $\Omega_{good}, \Omega_{bad}$
    \State $\cV_0 \gets \cV_0 \cup \{\Omega_{good}, \Omega_{bad}\}$
    \State $\cV_{queue} \gets \cV_{queue} \cup \{\Omega_{good}, \Omega_{bad}\}$
    \EndIf
    \EndWhile \\
    \Return $\cV_0$
    \end{algorithmic}
\end{algorithm}

\iffalse
\begin{algorithm}[bt]
    \caption{Partition Function}
    \begin{algorithmic}[1]
    \State {\bfseries Input:} Input Space $\Omega$, Samples $\cS_t$, Partitioning Latent Model $s(\vx)$.
    \State Set $\Omega \gets \Omega \cup \cS_t$
    \State Set $\cV_0 = \{ \Omega \}$
    \State Set $\cV_{queue} = \{ \Omega \}$
    \While{$\cV_{queue} \neq \emptyset$}
    \State $\Omega_p \gets \cV_{queue}[0]$
    \If{$s(\Omega_p)$ $k$-means splitable}
    \State $\Omega_{good}, \Omega_{bad} \gets k$-means$(s(\Omega_p))$
    \State $\cV_0 \gets \cV_0 \cup \Omega_{good} \cup \Omega_{bad}$
    \State $\cV_{queue} \gets \cV_{queue} \cup \Omega_{good} \cup \Omega_{bad}$
    \EndIf
    \EndWhile \\
    \Return $\cV_0$
    \end{algorithmic}
\end{algorithm}
\fi

\section{Proofs} \label{sec:proofs}
\subsection{Proof of Lemma~\ref{thm:max-bound}}
\begin{proof}
Let $\delta < 1$. Define the following cumulative density function (CDF):
\begin{equation}
    F_k(y) := \pr[f(\vx) \le g^*_k - y | \vx \in \Omega_k]
\end{equation}
where $g^*_k := \sup_{\vx\in \Omega_k}f(\vx)$. It is clear that $F_k(y)$ is a monotonically decreasing function with $F_k(0) = 1$ and $\lim_{y\rightarrow +\infty} F_k(y) = 0$. Here we assume it is strictly decreasing so that $F_k(y)$ has a well-defined inverse function $F_k^{-1}$. 

In the following, we will omit the subscript $k$ for brevity. Let us bound $\pr[g_t \ge g^* - y]$:
\begin{eqnarray}
    \pr[g_t \ge g^* - y] &=& 1 -\pr[g_t \le g^* - y] \\
    &\stackrel{\circle{1}}{=}& 1 - \prod_i \pr[f(\vx_i) \le g^* - y | \vx_i\in \Omega_k] \\
    &=& 1 - F_k^{n_t}(y) 
\end{eqnarray}
Note that $\circle{1}$ is due to the fact that all samples $\vx_1, \ldots, \vx_{n_t}$ are independently drawn within the region $\Omega_k$. Given $\delta$, let $r_t := F_k^{-1}(\delta^{1/n_t})$ and we have:
\begin{equation}
    \pr[g_t \ge g^* - r_t] = 1 - \delta 
\end{equation}
\end{proof}

\subsection{Proof of Corollary~\ref{corollary:lipschitz}}
\begin{proof}
Since $f$ is $L_k$-Lipschitz over region $\Omega_k$, we have:
\begin{equation}
    |f(\vx) - f(\vx')| \le L_k \|\vx-\vx'\|_2\quad \forall\vx,\vx'\in \Omega_k 
\end{equation}
Since the optimal solution $\vx^*_k \in \Omega_k$ is in the interior of $\Omega_k$, there exists $\epsilon_0$ so that $B(\vx_k^*, \epsilon_0) \subseteq \Omega_k$. From the Lipschitz condition, we know that in the ball $B(\vx^*_k, \epsilon)$ with $\epsilon \le \epsilon_0$, the function values are also quite good:
\begin{equation}
    f(\vx) \ge f(\vx^*_k) - L_k\|\vx-\vx^*_k\|_2 = g^* - L_k\epsilon, \quad\forall\vx \in B\left(\vx^*_k, \epsilon\right)
\end{equation}
Therefore, at least in the ball of $B(\vx_k^*, \epsilon)$, all function values are larger than a threshold $g^* - L_k\epsilon$. This means that for $\epsilon \le \epsilon_0$: 
\begin{equation}
    F_k(L_k\epsilon) = \pr\left[f(\vx) \le g^* - L_k\epsilon | \vx\in\Omega_k\right] \le 1 - \frac{V_0\epsilon^d}{V_k} 
\end{equation}
where $V_0$ is the volume of the unit $d$-dimensional sphere. Letting $\tilde V_k := V_k / V_0$ be the relative volume with respect to unit sphere, we have: 
\begin{equation}
    F_k(y) \le 1 - \frac{(y/L_k)^d}{\tilde V_k} = 1 - \left(\frac{y}{L_k\tilde V_k^{1/d}}\right)^d \quad \mathrm{when}\ y \le L_k\epsilon_0
\end{equation}
Therefore, $\Omega_k$ is at most $(1 - \epsilon^d_0\tilde V^{-1}_k, L_k\tilde V_k^{1/d})$-diluted with $z_k = 1 - \epsilon^d_0 \tilde V^{-1}_k$ and $c_k = L_k\tilde V_k^{1/d}$. 
\end{proof}

\iffalse
\begin{proof}
Since we do uniform sampling in $\Omega_k$, the probability that $\vx$ falls into a region is proportional to its volume. Therefore, we have:
\begin{equation}
    F(g-r) := \pr[f_t \le g-r] \le 1 - \frac{\vol(B(\vx_k^*, r/L_k))}{\vol(\Omega_k)} = 1 - \frac{V_0(r/L_k)^d}{V_k}
\end{equation}
\begin{equation}
    r_t := L_k\sqrt[d]{\frac{V_k}{V_0}\frac{1}{n_t}\ln\frac{1}{\delta}},
\end{equation}
then we have:
\begin{equation}
    \pr[g_t \ge g - r_t] = 1 - F^n(g-r_t) \ge 1 - \left[1 - \frac{V_0(r_t/L)^d}{V_k}\right]^n = 1 - \left(1-\frac{1}{n}\ln\frac{1}{\delta}\right)^n
\end{equation}

So we have $\pr[g_t \ge g - r_t] \ge 1 - \delta$. 
\end{proof}
\fi

\subsection{New Lemma and Proof}
\begin{lemma}
\label{lemma:fbound}
If $\Omega_k$ are $(z_k, c_k)$-diluted, then for any $\delta \in \left[z_k, 1\right]$ and $j \ge 1$, we have: 
\begin{equation}
    F_k^{-1}(\delta^{1/j}) \le c_k\sqrt[d]{\frac{1}{j}\ln \frac{1}{\delta}} 
\end{equation}
\end{lemma}
\begin{proof}
Note that the diluted condition $F_k(y) \le 1 - \left(\frac{y}{c_k}\right)^d$ for $y\in[0,c_k\sqrt[d]{1-z_k}]$ can be also be written as: 
\begin{equation}
F_k^{-1}(z) \le c_k\sqrt[d]{1-z},\quad \forall z \in [z_k, 1] \label{eq:finverse}
\end{equation}
Since now we have $z_k \le \delta \le \delta^{1/j} \le 1$ for any $j \ge 1$, following Eqn.~\ref{eq:finverse} we have: 
\begin{equation}
    F_k^{-1}(\delta^{1/j}) \le c_k\sqrt[d]{1 - \delta^{1/j}} 
\end{equation}
Due to the inequality that for $a < 1$ and $x > 0$, $a^x \ge 1 + x\ln a$ (which can be proven by simply showing the derivative is non-negative), if we take $a = \delta$ and $x = 1/j$, we have:
\begin{equation}
    \delta^{1/j} \ge 1 - \frac{1}{j}\ln\frac{1}{\delta}
\end{equation}
which gives:
\begin{equation}
    F_k^{-1}(\delta^{1/j}) \le c_k\sqrt[d]{j^{-1}\ln 1/\delta} 
\end{equation}
\end{proof}

\subsection{Proof of Theorem~\ref{thm:regret}}
\begin{proof}
Take $\delta = \eta / T^3$ so that $\delta \ge z_k$ for all regions $\Omega_k$. Then Eqn.~\ref{thm:max-bound} holds for all $T$ iterations and all $K$ arms with probability at least $1 - KT\delta$ by union bound, which we consider a ``good event''.

For brevity, define $g^*_k := g(\Omega_k)$ as the optimal function value within the region of $\Omega_k$ and $n_{k,t} := n_t(\Omega_k)$ the visitation count of region $\Omega_k$. Define $\Delta_k := f^* - g^*_k$ the minimal regret of each arm and $r_{k,t} := r_t(\Omega_k)$ the confidence bound. At iteration $t$, since we pick $k = a_t$ as the region to explore, it must be the case that:
\begin{equation}
    f^* + r_{k,t} \stackrel{\circle{1}}{\ge} g^*_k + r_{k,t} \stackrel{\circle{2}}{\ge} g_{k,t} + r_{k,t} \stackrel{\circle{3}}{\ge} g_{k^*,t} + r_{k^*,t} \stackrel{\circle{4}}{\ge} g^*_{k^*} = f^*
\end{equation}
where $k^*$ is the index of the optimal region $\Omega_{k^*}$ where its maximum $g^*_{k^*}$ is the global optimal value $f^*$. Here \circle{1} is due to global optimality of $f^*$, \circle{2} is due to global optimality of $g^*_k$ within region $\Omega_k$: $g^*_k \ge g_{k,t}$, \circle{3} is due to the fact that we pick $a_t=k$ at iteration $t$, and \circle{4} is due to the non-negativity of the confidence bound: $r_{k^*,t} \ge 0$. Therefore, since $g^*_k + r_{k,t} \ge f^*$, we have:
\begin{equation}
    \Delta_k := f^* - g_k^* \le r_{k,t} 
\end{equation}
Now we bound the total regret. 

Note that for $R_t(a_t) := f^* - g_{a_t,t}$, we have:
\begin{equation}
    R_t(a_t) := f^* - g_{a_t,t} = f^* - g^*_{a_t} + g^*_{a_t} - g_{a_t,t} \le 2r_{a_t, t}
\end{equation}
due to the fact that $\Delta_{k} = f^* - g^*_{k} \le r_{k,t}$ and the property of the confidence bound that $g_{k,t} \ge g^*_{k} - r_{k,t}$ with $k = a_t$. 

One the other hand, using Lemma~\ref{lemma:fbound}, we also have
\begin{equation}
    \Delta_k \le r_{k,t} = F_k^{-1}(\delta^{1/n_{k,t}}) \le c_k\sqrt[d]{\frac{1}{n_{k,t}}\ln \frac{1}{\delta}} 
\end{equation}
which means that 
\begin{equation}
    n_{k,t} \le \left(\frac{c_k}{\Delta_k}\right)^d\ln \frac{1}{\delta}
\end{equation}
So if the region $\Omega_k$ has a large gap $\Delta_k$, then $n_{k,t}$ would have a small upper-bound (and be small). As a result, we would never visit that region after a fixed number of visitations. This also helps bound the regret. 

If we sum over $R_t(a_t)$ over $t$ iterations, we get $R(T)$. We could reorganize them into two kinds of regions, the good regions where $\kgood := \{k: \Delta_k \le \Delta_0\}$ and the bad regions where $\kbad := \{k: \Delta_k > \Delta_0\}$:
\begin{equation}
    R(T) = \sum_{t=1}^T R_t(a_t) = \underbrace{\sum_{a_t \in \kgood} R_t(a_t)}_{R_\good(T)} + \underbrace{\sum_{a_t \in \kbad} R_t(a_t)}_{R_\bad(T)}  
\end{equation}
Let $M := \sup_{\vx\in\Omega} f(\vx) - \inf_{\vx\in\Omega} f(\vx)$ be the maximal gap between the highest and lowest function values. Note that $M$ is also the largest regret for a single move at any iteration. Letting $C_\bad := \left(\sum_{k\in \kbad} c_k^d\right)^{1/d}$ be the $\ell_d$-norm of $c_k$ over bad regions, we then have:  
\begin{eqnarray}
    R_\bad(T) &\le & M\left(\frac{C_\bad}{\Delta_0}\right)^d \ln\frac{1}{\delta} \\
    R_\good(T) &= & 2\sum_{k\in \kgood} \sum_{j=1}^{n_{k,T}} r_{k, t} \Big|_{n_{k,t} = j} = 2\sum_{k\in\kgood} \sum_{j=1}^{n_{k,T}} F_k^{-1}(\delta^{1/j})
\end{eqnarray}
For $R_\good(T)$, this is because for each region $k$ we visit it $n_{k,T}$ times and each time we pay a price that is proportional to $1 / n_{k, t}$ for $n_{k,t} = 1\ldots n_{k,T}$. 

Using Lemma~\ref{lemma:fbound}, since all $\Omega_k$ are $(z_k, c_k)$-concentrated and $z_k\le \delta$, this leads to:
\begin{equation}
    R_\good(T) \le 2\sqrt[d]{\ln 1/ \delta} \sum_{k\in \kgood} c_k \sum_{j=1}^{n_{k,T}} j^{-1/d} 
\end{equation}
Assuming $d > 1$ (high-dimensional case), we use the bound
\begin{equation}
\sum_{j=1}^n j^{-1/d} \le \frac{d}{d-1}n^{1-1/d}
\end{equation}
and we have:
\begin{equation}
     R_\good(T) \le \frac{2d}{d-1}\sqrt[d]{\ln 1/\delta} \sum_{k\in\kgood}  c_k n_{k,T}^{\frac{d-1}{d}} 
\end{equation}
H\"older's inequality says if $1/p+1/q=1$, then $\sum_k |x_k y_k| \le (\sum_k |x_k|^p)^{1/p} (\sum_k |y_k|^q)^{1/q}$. Using it with $p = d$ and $q = \frac{d}{d-1}$, we get
\begin{eqnarray}
     R_\good(T) &\le& \frac{2 d}{d-1}\sqrt[d]{\ln \frac{1}{\delta}} \left(\sum_{k\in\kgood} c_k^d\right)^{\frac{1}{d}}\left(\sum_{k\in\kgood} n_{k,T}\right)^{\frac{d-1}{d}} \\
     &\le& \frac{2 d}{d-1}\sqrt[d]{\ln \frac{1}{\delta}} C_\good T^{\frac{d-1}{d}} 
\end{eqnarray}
where $C_\good := \left(\sum_{k\in \kgood } c_k^d\right)^{\frac{1}{d}}$ is the $\ell_d$-norm of $c_k$ over good regions. 

Finally, if a good event doesn't happen (with probability $KT\delta$), we would pay a regret of at most $M$ at each iteration $t$, yield a bound of $M KT^2\delta$ for $T$ iterations.

Since $\delta = \eta / T^3$ then finally we have 
\begin{equation}
    \ee{R(T)} = \cO\left[C_\good \sqrt[d]{T^{d-1}\ln T} + M \left(\frac{C_\bad}{\Delta_0}\right)^d \ln T + KM\eta /T\right]
\end{equation}
\end{proof}

\iffalse
\begin{proof}
Note that from the definition of average Lipschitz constant $\bar L$, we have:
\begin{equation}
    \left(\sum_{k=1}^K L_k^d \tilde V_k\right)^{\frac{1}{d}} = \bar L \left(\frac{V}{V_0}\right)^{\frac{1}{d}}
\end{equation}
where $V := \sum_k V_k$ is the total volume, a constant quantity regardless of the partition and is \emph{independent} of the number of partition $K$. Define normalize total volume $\tilde V := V / V_0$ and we have: 
\begin{equation}
    R(T) \le \frac{2 \bar L d}{d-1}\sqrt[d]{3\ln T} \sqrt[d]{\tilde V} T^{\frac{d-1}{d}}
\end{equation}

\end{proof}
\fi

\iffalse
Let $\tilde V_k := V_k / V_0$ be the \emph{normalized} volume, plug in the expression of $r_{k,t}$ and we have:
\begin{equation}
    \Delta_k \le r_{k,t} = L_k \sqrt[d]{\frac{\tilde V_k}{n_{k,t}}\ln\frac{1}{\epsilon}} \label{eq:bad-arm-visitation}
\end{equation}
which means that
\begin{equation}
    n_{k,t} \le L_k^d \frac{\tilde V_k}{\Delta^d_k}\ln\frac{1}{\epsilon}
\end{equation}
Intuitively, this means that for bad arm with large $\Delta_k$, its visitation count is upper-bounded. 
\fi

\subsection{Additional Implications of Theorem 1}\label{sec:theorem_remarks}
\textbf{Relationship w.r.t sample complexity.} Note that one can turn the regret bound of $R(T)$ in Theorem~\ref{thm:regret} into sample complexity: if there exists $T$ such that $\ee{R(T)}/T \le \epsilon$, then with high probability there exists at least one $R_t(a_t) := f^* - g_t(\Omega_{a_t}) \le \epsilon$, showing that we already found a good $\vx \in \Omega_{a_t}$ with $f(\vx) = g_t(\Omega_{a_t}) \ge f^* - \epsilon$. To achieve this, since $R(T) \sim T^{\frac{d-1}{d}}$, we set $R(T)/T \sim T^{-\frac{1}{d}}$. Then the sample complexity $T$ to achieve the \emph{global optimum} within an $\epsilon$-ball is $\sim 1/\epsilon^d$, which is the best we can achieve without structured information on $f$. Previous papers~\cite{slivkins2019introduction} show a slightly worse bound $\cO(T^{\frac{d+1}{d+2}})$ since they also consider stochastic functions and discretization error.

\textbf{Which region to split?} Since $C_\good := \left(\sum_{k\in \kgood} c_k \right)^{1/d}$ is an $\ell_d$-norm, when $d$ is large (i.e., high-dimensional), $C_\good \sim \max_{k\in \Omega_\good} c_k$ so ideally we should split the region with the highest $c_k$ to reduce $C_\good$ the most. Intuitively this means the most diluted / scattered region.  

% \input{alternative_theorem}

% empty

\section{Model-Based Reinforcement Learning}
\ours{} can escape local minima and achieve significantly better results in various RL tasks using a simulated environment. In MiniWorld, we showed that we could also plug \ours{} into the PETS~\cite{chua2018deep} framework, replacing the CEM method that was originally used as a planner (Sec. \ref{sec:synthetic_experiments}).
%One natural question to ask is whether \ours{} can work with a learnable model and act as a plug-in planner? 
% We show the performance by just plug in \ours{} in the PETS~\citep{chua2018deep} framework replacing the CEM method that was originally used as a planner. 
Here we additionally use Mujoco, a commonly used benchmark, to validate the performance. Note that Mujoco is a very smooth task and doesn't contain many local minima, so traditional methods work reasonably well in this domain. In Tab.~\ref{tab:Mujoco_results2}, we can see that in easier tasks like Reacher and Pusher, \ours{} is a little worse than CEM. However, in hard tasks like Halfcheetah and Walker, \ours{} has over 1000 reward gain over baseline methods.

\begin{table*}[!htb]
\caption{Results for Mujoco with replanning frequency of 5. We see that \ours{} performs substantially better than CEM and RS in hard tasks like Halfcheetah and Walker.}
\footnotesize
\setlength\tabcolsep{3.5pt}
\label{tab:Mujoco_results2}
\centering
\begin{tabular}{@{}p{1.75cm}p{2.1cm}p{1.75cm}p{2.5cm}p{2cm}p{2.4cm}@{}}
% {@{}lccccc@{}}
\toprule
& \textbf{swimmer} & \textbf{acrobot} & \textbf{hopper} & \textbf{pendulum} & \textbf{halfcheetah} \\ 
\midrule  
PETS(RS) & \textbf{12.92}$\pm$\textbf{7.92} & -41.93$\pm$2.17 & -1525.39$\pm$222.43 & 130.14$\pm$28.39 & 497.03$\pm$121.72  \\ 
PETS(CEM) & -6.87$\pm$1.30  & -24.77$\pm$7.63 & -2102.57$\pm$136.35 & 153.05$\pm$12.00 & 271.01$\pm$165.08\\
\ours{} & 10.82$\pm$8.47 & \textbf{6.29}$\pm$\textbf{10.29} & \textbf{-1205.01}$\pm$\textbf{167.52} & \textbf{153.70}$\pm$\textbf{38.02} & \textbf{3942.47}$\pm$\textbf{400.01} \\
\bottomrule 
\end{tabular}
\centering
\begin{tabular}{@{}p{1.75cm}p{2.1cm}p{1.75cm}p{2.5cm}p{2cm}p{2.4cm}@{}}
% {@{}lccccc@{}}
\toprule
& \textbf{reacher} & \textbf{pusher} & \textbf{ant} & \textbf{I-pendulum} & \textbf{walker} \\ 
\midrule  
PETS(RS) & -1165.59$\pm$12.04 & -220.58$\pm$2.94 & 1330.81$\pm$113.17 & -11.87$\pm$10.43 & -1204.94$\pm$344.70 \\
PETS(CEM) & \textbf{-36.45}$\pm$\textbf{2.87} &  \textbf{-90.70}$\pm$\textbf{7.29} & \textbf{1405.56}$\pm$\textbf{46.94} & -4.84$\pm$5.29 & -2036.28$\pm$213.41 \\
\ours{} & -40.31$\pm$5.05 & -103.42$\pm$2.91 & 1033.46$\pm$148.87 & \textbf{-0.30}$\pm$\textbf{0.09} & \textbf{-53.25}$\pm$\textbf{987.53}\\
\bottomrule 
\end{tabular}
\end{table*}

\section{Evaluation on Synthetic Functions}

We additionally evaluate \ours{} on some synthetic functions (Ackley and Levy functions, both 20-dimensional and 100-dimensional) used in the original LaMCTS paper \cite{wang2020learning}. For these tasks we compare to just the original LaMCTS method. Both \ours{} and LaMCTS use the TuRBO inner solver following \cite{wang2020learning} for the 20-dimensional version of both functions, and the CMA-ES inner solver for the 100-dimensional version for computational efficiency. \ours{} performs equal or better on these tasks (Table \ref{tab:synthetic_functions}; note lower is better).

\begin{table*}[!htb]

\footnotesize
\setlength\tabcolsep{3.5pt}
\centering
\begin{tabular}%{p{2cm}p{2cm}p{2cm}p{2.5cm}p{2cm}p{2cm}p{2cm}}
{lcccc}
\toprule
& \textbf{Ackley-20D} &\textbf{Levy-20D}  & \textbf{Ackley-100D}& \textbf{Levy-100D} \\ 
\midrule  
LaMCTS & 0.48 $\pm$ 0.03&	0.51 $\pm$ 0.09&	0.65 $\pm$ 0.25&	14.24 $\pm$ 4.87 \\
\ours{} & 0.49 $\pm$ 0.04&	0.34 $\pm$ 0.07&	0.46 $\pm$ 0.15&	11.95 $\pm$ 3.56\\
\bottomrule 
\end{tabular}
\caption{\small \ours{} vs the original LaMCTS method on some synthetic functions evaluated in the original LaMCTS work. Note \textit{lower} is better. \ours{} performs equal or better on these tasks. 
\label{tab:synthetic_functions}
%\yuandong{Putting Table 1 and Table 2 side by side?}
}
\end{table*}

\section{Tables of Numerical Results}\label{sec:numerical_tables}

We provide in Tables~\ref{tab:miniworld_results_oracle} through~\ref{tab:molecule} the numerical final rewards for our tasks, corresponding to the plots in the main text. 

%\subsection{MiniWorld}

\begin{table*}[!htb]

\footnotesize
\setlength\tabcolsep{3.5pt}
\centering
\begin{tabular}{p{2cm}p{2cm}p{2cm}p{2.5cm}p{2cm}p{2cm}p{2cm}}
% {lcccccccccccc}
\toprule
& \textbf{MazeS3} & \textbf{FourRooms}  & \textbf{SelectObj} \\ 
\midrule  
LaMCTS & 23.4 $\pm$ 2.6	& 20.3 $\pm$ 2.5 & \phantom{0}0.8 $\pm$ 0.6 \\
RS & \phantom{0}0.4 $\pm$ 0.4& \phantom{0}0.0 $\pm$ 0.0  & \phantom{0}3.1 $\pm$ 1.1 \\
CMA-ES & 23.8 $\pm$ 2.7& 79.7 $\pm$ 2.5  & \phantom{0}1.2 $\pm$ 0.7   \\
CEM & 25.0 $\pm$ 2.7 & 69.9 $\pm$ 2.9  & \phantom{0}0.4 $\pm$ 0.4  \\
VOOT & 26.2 $\pm$ 2.7  &\phantom{0}0.0 $\pm$ 0.0 & \phantom{0}0.0 $\pm$ 0.0 \\
RandDOOT & 25.0 $\pm$ 2.7  &\phantom{0}0.0 $\pm$ 0.0 & \phantom{0}0.0 $\pm$ 0.0 \\
iLQR & \phantom{0}3.1 $\pm$ 1.1 & \phantom{0}0.8 $\pm$ 0.6  & \phantom{0}1.6 $\pm$ 0.8\\
PPO & \phantom{0}0.0 $\pm$ 0.0 & \phantom{0}0.0 $\pm$ 0.0  & 31.3 $\pm$ 8.2\\
\ours{} & \textbf{57.0 $\pm$ 3.1} & \textbf{89.1 $\pm$ 2.0}  & \textbf{45.3 $\pm$ 3.1} \\
% \ours{} latent & 0.840 & 0.500 & \\
\bottomrule 
\end{tabular}
\caption{\small Results (success percentage over 256 trials) for MiniWorld tasks for different methods, querying oracle transition model. \ours{} substantially outperforms all baselines on all three environments. This table corresponds to Figure \ref{fig:miniworld_graph}.
% \yuandong{Specify how many evaluations are used?}
\label{tab:miniworld_results_oracle}
%\yuandong{Putting Table 1 and Table 2 side by side?}
}
\end{table*}

\begin{table*}[!htb]
\RawFloats
% \caption{Results for MiniWorld PETS.}
\footnotesize
\setlength\tabcolsep{3.5pt}
\label{tab:miniworld_results_pets}
\centering
\begin{tabular}{p{2cm}p{2cm}p{2cm}p{2.5cm}p{2cm}p{2cm}p{2cm}}
% {lcccccccccccc}
\toprule
& \textbf{FourRooms} & \textbf{SelectObj} \\ 
\midrule
PETS-RS & \phantom{0}0.0 $\pm$ 0.0 & \phantom{0}0.0 $\pm$ 0.0 \\
PETS-CEM & 66.9 $\pm$ 6.3 & \phantom{0}7.2 $\pm$ 1.9  \\
PETS-\ours{} & \textbf{83.1 $\pm$ 2.3} & \textbf{19.4 $\pm$ 1.8}  \\
\bottomrule 
\end{tabular}
\caption{\small Results for MiniWorld tasks for different methods using a learned PETS transition model. The oracle model is only used for final evaluation on each environment seed, and the resulting trajectory becomes future training data to the PETS model. We report the total fraction of environment seeds solved by each method, out of 256 total, averaged across 5 trials. \ours{} substantially outperforms the original PETS implementation. This table corresponds to Figure \ref{fig:miniworld_pets_graph}. 
% \yuandong{Specify how many evalutions are used?}
}
\end{table*}

%\subsection{Region Selection Analysis}

\begin{table*}[!htb]
\footnotesize
\setlength\tabcolsep{3.5pt}
\centering
\begin{tabular}{p{3cm}p{2cm}p{2cm}p{2.5cm}p{2cm}p{2cm}p{2cm}}
% {lcccccccccccc}
\toprule
 & \textbf{MazeS3} & \textbf{FourRooms} & \textbf{SelectObj} \\ 
\midrule  
% CEM & 69.9 $\pm$ 2.9 & 25.0 $\pm$ 2.7 & \phantom{0}0.4 $\pm$ 0.4  \\
\ours{} &  \textbf{57.0 $\pm$ 3.1} &\textbf{89.1 $\pm$ 2.0} & \textbf{45.3 $\pm$ 3.1} \\
\ours{}-$\mathrm{mean}$ &  43.4 $\pm$ 3.1 &83.6 $\pm$ 2.3 & 16.0 $\pm$ 2.3 \\
\ours{}-$\mathrm{nolatent}$ & 26.6 $\pm$ 2.8&	54.3 $\pm$ 3.1&	\phantom{0}2.3 $\pm$ 0.9 \\
\ours{}-$\mathrm{notree}$  & 24.6 $\pm$ 2.7 & \phantom{0}0.0 $\pm$ 0.0& \phantom{0}2.7 $\pm$ 1.0 \\
\ours{}-$\mathrm{noUCB}$ & 28.1 $\pm$ 2.8  & 76.6 $\pm$ 2.6& \phantom{0}1.6 $\pm$ 0.8 \\
\bottomrule 
\end{tabular}
\caption{\small Results for MiniWorld tasks using different region selection methods. 256 trials per method. This corresponds to Figure \ref{fig:region_select}.
% \yuandong{How many evaluations?}
%CEM reproduced for baseline comparison. 
}
\label{tab:leaf_selection}
\end{table*}

%\subsection{Compiler Phase Ordering}
\begin{table*}[!htb]

\footnotesize
\setlength\tabcolsep{3.5pt}
\centering
\begin{tabular}{p{1.7cm}p{1cm}p{1cm}p{1.3cm}p{1.3cm}p{1cm}p{1cm}p{1cm}p{1cm}p{1cm}}
% {lcccccccccccc}
\toprule
& \textbf{adpcm} & \textbf{aes} & \textbf{blowfish} & \textbf{dhrystone} & \textbf{gsm} & \textbf{matmul} & \textbf{ mpeg2} & \textbf{qsort} & \textbf{sha} \\ 
\midrule  
-O0 & 41260 & 12633 & 199345 & 9258 & 8130 & 42085 & 10489 & 58400 & 269653 \\
-O3 & 16844 & 9937 & 188237 & 5936 & 7137 & 33244 & 8266 & 52256 & 226235 \\
PPO\_50 & 11175 & 10263 & \textbf{175649} & 5753 & 6286 & \textbf{9644} & 8281 & 52137 & 209142 \\
OpenTuner & 10501 & 9795 & 180834 & 7196 & 6181 & 33244 & 8291 & 52137 & 209155 \\
CMA-ES & \textbf{10451} & 10093 & 180198 & 5996 & 6294 & \textbf{9644} & 8280 & 50869 & 209142\\
\ours{} & \textbf{10451} & \textbf{9753} & 180179 & 5702 & \textbf{6178} & \textbf{9644} & 8282 & \textbf{47745} & 209142 \\
\hline 
PPO\_4000 & 10415 & 9759 & 175779 & \textbf{5515} & 6286 & \textbf{9644} & \textbf{8260} & 47785 & \textbf{205302} \\
\bottomrule 
\end{tabular}
\caption{\small Results for compiler phase ordering, in execution cycles for program after a series of transformations, following setup of \cite{haj2020autophase}. 50 oracle accesses per method. This table corresponds to Figure \ref{fig:autophase}. 
%\yuandong{How many evaluations?}
}
\label{tab:phase_order}
\end{table*}

\FloatBarrier

%\subsection{Molecular Design}

\begin{table*}[htbp]
% \caption{Results for molecules}
\footnotesize
\setlength\tabcolsep{3.5pt}
\centering
\begin{tabular}
% {p{2cm}p{2cm}p{2cm}p{2cm}p{2cm}p{2cm}p{2cm}}
{lcccc}
\toprule
& \textbf{QED} & \textbf{DRD2} & \textbf{HIV} & \textbf{SARS}\\ 
\midrule  
LaMCTS & 0.914 $\pm$ 0.002&	0.323 $\pm$ 0.016&	0.406 $\pm$ 0.019&	0.452 $\pm$ 0.010\\
RS & 0.897 $\pm$ 0.001 & 0.081 $\pm$ 0.006 & 0.116 $\pm$ 0.002 & 0.279 $\pm$ 0.006 \\
CEM & 0.906 $\pm$ 0.003 & 0.250 $\pm$ 0.016 & 0.455 $\pm$ 0.021 & 0.423 $\pm$ 0.005 \\
CMA-ES & 0.888 $\pm$ 0.004 & 0.216 $\pm$ 0.018 & 0.425 $\pm$ 0.020 & 0.414 $\pm$ 0.007 \\
\ours{} & \textbf{0.916 $\pm$ 0.001} & \textbf{0.648 $\pm$ 0.026} & \textbf{0.588 $\pm$ 0.020} & \textbf{0.570 $\pm$ 0.018} \\
\bottomrule 
\end{tabular}
\caption{\small Mean and standard deviation across 128 random seeds for \ours{} and baselines on QED, DRD2, SARS, and HIV molecular design tasks; results reported for 1000, 4000, 4000, and 4000 oracle queries respectively. \ours{} significantly outperforms all three baselines on all four properties. This table corresponds to Figure \ref{fig:molecule_graph}.}
\label{tab:molecule}
\end{table*}

\section{Detailed Analyses and Ablations}\label{sec:detailed_analysis}

\subsection{$L_k$ and $c_k$ Estimation Details}\label{sec:lipschitz_estimation}

To loosely approximate the Lipschitz constant in our analysis from Sec. \ref{sec:analysis}, we simply check all pairwise Lipschitz constants between existing samples (candidate trajectories) in the tree node (region). Similarly, to loosely approximate $c_k$, we take the highest-scoring sample in the region as the ``optimum'' and estimate $c_k$ for $z_k=0.5$ following our definition using the remaining samples in the region. %\yuandong{The results are presented in which section?}

\subsection{$z_k$ Estimation}\label{sec:empirical_zk}

We estimate $z_k$ over time in our MiniWorld tasks, fixing several different values of $c_k$ in intervals of 1 reward (Figure \ref{fig:zk}). $z_k$ is estimated using 50 samples (in between each dynamic re-partitioning of the space) at each timestep in intervals of 50, with 32 trial runs. In all cases $z_k$ initially drops very quickly, and then somewhat plateaus after finding the initial local optimum (whether global or not), especially in SelectObj. However, in most cases it still continues to decrease over time. 

While this is consistent with our qualitative analysis in Sec.~\ref{sec:theorem_remarks} about how $z_k$ changes with recursive splitting, in some cases, $z_k$ seems to stop decreasing over time. Upon inspection, we find that those regions whose $z_k$ remain high correspond to low-performing regions which do not receive many samples according to UCB exploration. Therefore, such regions won't improve over time and the $z_k$ remains high. 

%For example, why in case cases, $z_k$ stop decreasing? Is that because these sub-regions have few visitation due to poor performance, and they stop improving themselves? \yuandong{Ideally we should answer them in this supp material}. We leave it for future work.  

% \yuandong{Interesting. Does that mean we should split more frequently in the lower nodes? Or the splitting is not as effective when the tree goes deeper? I remembered that in LaMCTS code there is a maximum depth setting somewhere, maybe that's the reason?} \kevin{Actually, now that I think about it, the ones that are not decreasing as much over time (the grey and pink curves) are the ones corresponding to a very high cutoff---to beat the cutoff you have to be near-optimal. Part of the point of \ours{} is that we still spend samples on suboptimal regions as part of exploration, though? Maybe we can discuss tomorrow, it could be I misunderstood what exactly we are calculating here}

\begin{figure*}[htbp]
    \centering
    \begin{subfigure}[t]{0.32\textwidth}
        \centering
        \includegraphics[height=1.4in]{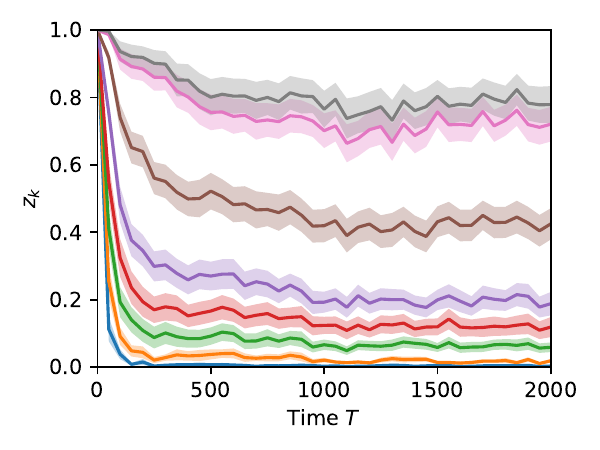}
        \vspace{-20pt}
        \caption{MazeS3}
    \end{subfigure}%
    \hfill
    \begin{subfigure}[t]{0.32\textwidth}
        \centering
        \includegraphics[height=1.4in]{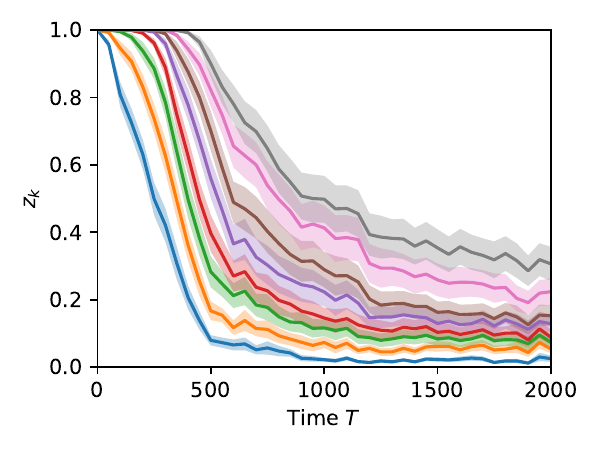}
        \vspace{-20pt}
        \caption{FourRooms}
    \end{subfigure}
    \hfill
    \begin{subfigure}[t]{0.32\textwidth}
        \centering
        \includegraphics[height=1.4in]{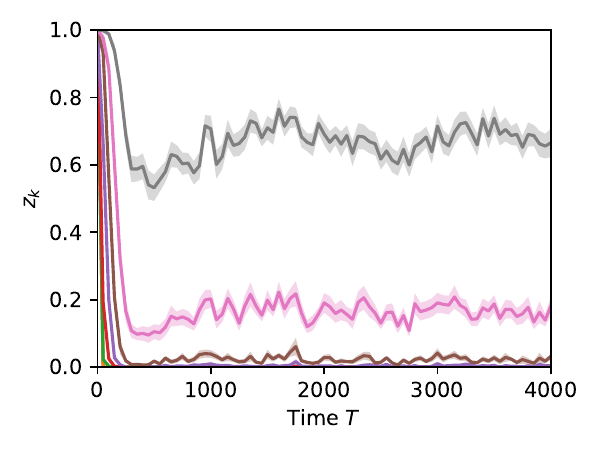}
        \vspace{-20pt}
        \caption{SelectObj}
    \end{subfigure}
    \vspace{-5pt}
    \caption{\small Mean and standard deviation (32 trials), of estimated $z_k$ for different values of $c_k$ on MiniWorld tasks.
    }
    \label{fig:zk}
\end{figure*}

\subsection{Latent Space Visualization}

We show a t-SNE visualization (Figure \ref{fig:tsne_mazes3}) of the latent space of trajectories at the end of a sample MazeS3 run of \ours{}. The first sampled trajectories are colored red, with a gradient toward blue for the later-sampled trajectories. The later trajectories are clearly separated in the latent space. 

\begin{figure*}[htbp]
    \centering
    \includegraphics[height=2in]{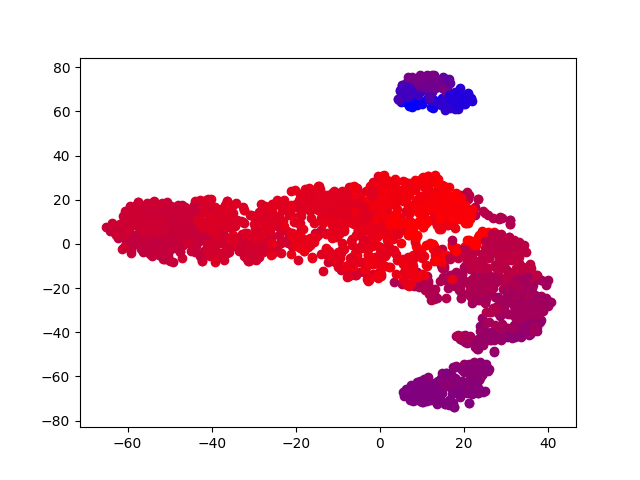}
    \caption{\small Latent space t-SNE visualization for a sample MazeS3 run of \ours{}. Earlier-sampled trajectories (red) are clearly separated from the latest-sampled trajectories (blue). 
    }
    \label{fig:tsne_mazes3}
\end{figure*}

\subsection{Parameter Space Methods}
It is of course possible to optimize the parameters of a policy which outputs an action given the current state, as in the original LaMCTS formulation, or in PPO. Nevertheless, we tune and run a parameter-space version of LaMCTS in the MiniWorld tasks, which is essentially the original LaMCTS adapted to path planning, only with TuRBO replaced with CMA-ES as in \ours{} due to speed considerations. Specifically, following LaMCTS, we learn the parameters of a linear policy for outputting actions given states. 

\begin{table*}[!htb]
\footnotesize
\setlength\tabcolsep{3.5pt}
\centering
\begin{tabular}{p{3cm}p{2cm}p{2cm}p{2.5cm}p{2cm}p{2cm}p{2cm}}
% {lcccccccccccc}
\toprule
& \textbf{MazeS3}& \textbf{FourRooms}  & \textbf{SelectObj} \\ 
\midrule  
LaMCTS-parameter & 10.9 $\pm$ 1.9 & \phantom{0}9.4 $\pm$ 1.8  & 100.0 $\pm$ 0.0 \\
PPO & \phantom{0}0.0 $\pm$ 0.0 & \phantom{0}0.0 $\pm$ 0.0   & 31.3 $\pm$ 8.2 \\
	%0.0 $\pm$ 0.0	0.0 $\pm$ 0.0	31.3 $\pm$ 8.2
\ours{} & 57.0 $\pm$ 3.1 & 89.1 $\pm$ 2.0  & \phantom{0}45.3 $\pm$ 3.1 \\
\bottomrule 
\end{tabular}
\caption{\small Comparison of \ours{} to an adaptation of the original LaMCTS, operating in \textit{parameter} space, on MiniWorld tasks. SelectObj is uniquely advantageous to parameter-space methods; for this reason, PPO also performs better than our other baselines (but worse than \ours{}) on that environment only. 256 trials per method.}
\label{tab:miniworld_parameter}
\end{table*}

Working in the parameter space can be clearly advantageous when states are relatively simple and low-dimensional, as in the Mujoco environments evaluated in LaMCTS and POPLIN~\cite{wang2019exploring}, or when the policy barely needs to depend on the state at all, as in our SelectObj task (Table \ref{tab:miniworld_parameter}). We designed the SelectObj task as a challenge for path planning algorithms operating in the action space, which struggle to escape the local optimum, but in truth this environment can be trivially solved by simply moving in the same correct direction at every step (toward the far goal). 

On the other hand, more complex policies may be more challenging to learn when the state representation is higher-dimensional, which may be the case in practical tasks. This is the case in our MazeS3 and FourRooms environments, where the state is represented as a top-down image rather than a vector of position and velocity information. Unlike SelectObj, these tasks require navigation around obstacles rather than just moving in a straight line. Despite featurizing with the same randomly initialized CNN as \ours{}, LaMCTS-parameter performs very poorly on MazeS3 and FourRooms in comparison. Additionally, methods like LaMCTS-parameter which use a parameter space must critically depend on the specific parametric form of the policy to be learned (e.g., whether it is a linear policy, a nonlinear policy parameterized by neural networks, etc); therefore, it is not obvious how to take advantage of a latent space which encodes a \textit{sequence} of states and/or actions, which is critical in environments such as our molecular design tasks. 

\subsection{Hyperparameter Sensitivity Analysis}\label{sec:hyperparameter_sensitivity}

Since the $C_p$ parameter is the only additional parameter we tune in \ours{}, we analyze the sensitivity of \ours{}'s performance with respect to $C_p$ on the MiniWorld tasks. Our main results use $C_p=2$ except for SelectObj where we used $C_p=4$; here we run $C_p=1,2,4$ for all three tasks and show the results in Table \ref{tab:miniworld_cp}. \ours{} even with poorly tuned $C_p$ values still substantially outperforms CEM on these tasks with difficult-to-escape local optima.

\begin{table*}[!htb]

\footnotesize
\setlength\tabcolsep{3.5pt}
\centering
\begin{tabular}{p{3cm}p{2cm}p{2cm}p{2.5cm}p{2cm}p{2cm}p{2cm}}
% {lcccccccccccc}
\toprule
& \textbf{MazeS3} & \textbf{FourRooms}  & \textbf{SelectObj} \\ 
\midrule  
CEM & 25.0 $\pm$ 2.7 & 69.9 $\pm$ 2.9  & \phantom{0}0.4 $\pm$ 0.4  \\
\ours{}-$C_p=1$ & 52.7 $\pm$ 3.1 & 89.5 $\pm$ 1.9  & \phantom{0}6.6 $\pm$ 1.6\\
\ours{}-$C_p=2$ & 57.0 $\pm$ 3.1 & 89.1 $\pm$ 2.0  & 23.1 $\pm$ 2.6 \\
\ours{}-$C_p=4$ & 53.5 $\pm$ 3.1 & 87.1 $\pm$ 2.1  & 45.3 $\pm$ 3.1 \\
\bottomrule 
\end{tabular}
\caption{\small Results for MiniWorld tasks using different $C_p$ values for \ours{}. $C_p=2$ corresponds to our main paper results, except for SelectObj where we used $C_p=4$. \ours{} is relatively insensitive to changes in $C_p$ on MazeS3 and FourRooms and only more sensitive on the more difficult SelectObj task. However, even poorly tuned versions of \ours{} outperform CEM, reproduced for baseline comparison. 256 trials per method.}
\label{tab:miniworld_cp}
\end{table*}

\subsection{Max vs. Mean UCB Metric For MCTS}\label{sec:max_vs_mean}

Our theory suggests that the UCB metric for MCTS should be based on the max function value rather than the mean for the deterministic functions that we consider in this work. Figure \ref{fig:region_select} already shows this for MiniWorld; here we show the max vs. mean analysis for all tasks in Tables \ref{tab:miniworld_mean}, \ref{tab:minigrid_mean}, \ref{tab:autophase_mean}, and \ref{tab:molecule_mean}. 

\begin{table*}[!htb]
\footnotesize
\setlength\tabcolsep{3.5pt}
\centering
\begin{tabular}{p{3cm}p{2cm}p{2cm}p{2.5cm}p{2cm}p{2cm}p{2cm}}
% {lcccccccccccc}
\toprule
& \textbf{MazeS3}& \textbf{FourRooms}  & \textbf{SelectObj} \\ 
\midrule  
\ours{}-$\mathrm{mean}$ & 43.4 $\pm$ 3.1&83.6 $\pm$ 2.3&	16.0 $\pm$ 2.3 \\
\ours{} & 57.0 $\pm$ 3.1 & 89.1 $\pm$ 2.0  & 45.3 $\pm$ 3.1 \\
\bottomrule 
\end{tabular}
\caption{\small Results for \ours{} (using max function value for UCB) in MiniWorld compared to \ours{} using the mean function value metric for UCB. \ours{} is substantially better. 256 trials per method.}
\label{tab:miniworld_mean}
\end{table*}

\begin{table*}[!htb]
\caption{\small Results for \ours{} (using max function value for UCB) in MiniGrid compared to \ours{} using the mean function value metric for UCB. \ours{} performs similarly or better. 256 trials per method.}
\footnotesize
\setlength\tabcolsep{3pt}
\label{tab:minigrid_mean}
\centering
\begin{tabular}{p{2cm}p{1.5cm}p{1.5cm}p{1.7cm}p{1.7cm}p{2cm}p{2cm}}
% {@{}lcccccccccccc@{}}
\toprule
& \textbf{DK-6} & \textbf{DK-8} & \textbf{KC-S3R3} & \textbf{KC-S3R4} & \textbf{MR-N4S5} & \textbf{MR-N6}\\ 
\midrule
\ours{}-$\mathrm{mean}$ & 0.98 $\pm$ 0.02	&0.25 $\pm$ 0.13	&-2.36$\pm$0.09&	-4.36$\pm$0.12	&-11.78 $\pm$ 0.77&	-114.63 $\pm$ 4.53 \\
\ours{} & 0.95$\pm$0.03 & 0.46$\pm$0.13 & -2.27$\pm$0.09 & -4.37$\pm$0.13 & -11.68$\pm$0.75 & -113.53$\pm$4.49  \\

\bottomrule 
\end{tabular}
\end{table*} 

\begin{table*}[!htb]

\footnotesize
\setlength\tabcolsep{3.5pt}
\centering
\begin{tabular}{p{1.7cm}p{1cm}p{1cm}p{1.3cm}p{1.3cm}p{1cm}p{1cm}p{1cm}p{1cm}p{1cm}}
% {lcccccccccccc}
\toprule
& \textbf{adpcm} & \textbf{aes} & \textbf{blowfish} & \textbf{dhrystone} & \textbf{gsm} & \textbf{matmul} & \textbf{ mpeg2} & \textbf{qsort} & \textbf{sha} \\ 
\midrule  
\ours{}-$\mathrm{mean}$ & 10501 & 10407 & 176429 & 5740 & 6305 & 8841 & 8281 & 47745 & 209142 \\
\ours{} & 10451 & 9753 & 180179 & 5702 & 6178 & 9644 & 8282 & 47745 & 209142 \\
\bottomrule 
\end{tabular}
\caption{\small Results for \ours{} (using max function value for UCB) in compiler phase ordering compared to \ours{} using the mean function value metric for UCB. The two versions perform similarly. 256 trials per method. 
}
\label{tab:autophase_mean}
\end{table*}

\begin{table*}[!htb]
\footnotesize
\setlength\tabcolsep{3.5pt}
\centering
\begin{tabular}{p{3cm}p{2cm}p{2cm}p{2.5cm}p{2cm}p{2cm}p{2cm}}
% {lcccccccccccc}
\toprule
& \textbf{QED}& \textbf{DRD2}  & \textbf{HIV} & \textbf{SARS} \\ 
\midrule  
\ours{}-$\mathrm{mean}$ & 0.914 $\pm$ 0.002&	0.323 $\pm$ 0.016&	0.406 $\pm$ 0.019&	0.452 $\pm$ 0.010 \\
\ours{} & 0.916 $\pm$ 0.001&	0.648 $\pm$ 0.026&	0.588 $\pm$ 0.020&	0.570 $\pm$ 0.018\\

\bottomrule 
\end{tabular}
\caption{\small Results for \ours{} (using max function value for UCB) in molecular design tasks compared to \ours{} using the mean function value metric for UCB. \ours{} performs similarly on the easiest QED task and substantially better on the others. 256 trials per method.}
\label{tab:molecule_mean}
\end{table*}

\subsection{Latent Space Ablations}\label{sec:latent_space_ablations}

We conduct additional analysis on the use of a latent partition space $\Phi_s$ in the MiniWorld, MiniGrid, and compiler phase ordering tasks in Tables \ref{tab:phis_miniworld} (reproduced from Table \ref{tab:leaf_selection}), \ref{tab:phis_minigrid}, and \ref{tab:phis_autophase} respectively. \ours{} performs similarly or better compared to the version without a latent space; the difference is especially large in MiniWorld. 

\begin{table*}[!htb]
\footnotesize
\setlength\tabcolsep{3.5pt}
\centering
\begin{tabular}{p{3cm}p{2cm}p{2cm}p{2.5cm}p{2cm}p{2cm}p{2cm}}
% {lcccccccccccc}
\toprule
& \textbf{MazeS3}& \textbf{FourRooms}  & \textbf{SelectObj} \\ 
\midrule  
\ours{}-$\mathrm{nolatent}$ & 26.6 $\pm$ 2.8 & 54.3 $\pm$ 3.1  & \phantom{0}2.3 $\pm$ 0.9 \\

\ours{} & 57.0 $\pm$ 3.1 & 89.1 $\pm$ 2.0  & \phantom{0}45.3 $\pm$ 3.1 \\
\bottomrule 
\end{tabular}
\caption{\small Results for \ours{} in MiniWorld compared to \ours{} without the use of a partition latent space $\Phi_s$. \ours{} is substantially better. 256 trials per method.}
\label{tab:phis_miniworld}
\end{table*}

\begin{table*}[!htb]
\caption{\small Results for \ours{} in MiniGrid compared to \ours{} without the use of a partition latent space $\Phi_s$. \ours{} is better in most cases. 256 trials per method.}
\footnotesize
\setlength\tabcolsep{3pt}
\label{tab:phis_minigrid}
\centering
\begin{tabular}{p{2.5cm}p{1.5cm}p{1.5cm}p{1.5cm}p{1.5cm}p{2cm}p{2cm}}
% {@{}lcccccccccccc@{}}
\toprule
& \textbf{DK-6} & \textbf{DK-8} & \textbf{KC-S3R3} & \textbf{KC-S3R4} & \textbf{MR-N4S5} & \textbf{MR-N6}\\ 
\midrule
\ours{}-$\mathrm{nolatent}$ & 0.98$\pm$0.02 & 0.25$\pm$0.13 & -2.36$\pm$0.09 & -4.36$\pm$0.12 & -11.78$\pm$0.77 & -114.63$\pm$4.53 \\
\ours{} & 0.95$\pm$0.03 & 0.46$\pm$0.13 & -2.27$\pm$0.09 & -4.37$\pm$0.13 & -11.68$\pm$0.75 & -113.53$\pm$4.49  \\

\bottomrule 
\end{tabular}
\end{table*} 

\begin{table*}[!htb]

\footnotesize
\setlength\tabcolsep{3.5pt}
\centering
\begin{tabular}{p{1.7cm}p{1cm}p{1cm}p{1.3cm}p{1.3cm}p{1cm}p{1cm}p{1cm}p{1cm}p{1cm}}
% {lcccccccccccc}
\toprule
& \textbf{adpcm} & \textbf{aes} & \textbf{blowfish} & \textbf{dhrystone} & \textbf{gsm} & \textbf{matmul} & \textbf{ mpeg2} & \textbf{qsort} & \textbf{sha} \\ 
\midrule  
\ours{}-$\mathrm{nolatent}$ & 10451 & 10263 & 176429 & 6617 & 6169 & 8841 & 8280 & 52137 & 476269 \\
\ours{} & 10451 & 9753 & 180179 & 5702 & 6178 & 9644 & 8282 & 47745 & 209142 \\
\bottomrule 
\end{tabular}
\caption{\small Results for \ours{} in compiler phase ordering compared to \ours{} without the use of a partition latent space $\Phi_s$. The two versions are comparable in most cases. 
}
\label{tab:phis_autophase}
\end{table*}

Additionally, it is possible to use a separate sampling latent space $\Phi_h$, as illustrated here in MiniGrid (Table \ref{tab:phih}), although we do not do so in our main results to keep consistency between latent spaces across tasks. The version with a latent space (a reversible flow model here) performs slightly better. %Although we found that the separate sampling space $\Phi_h$ did not significantly affect performance in MiniWorld, we show ablations on $\Phi_h$ in MiniGrid here as well. The partition latent space makes a large difference in MiniWorld (Table \ref{tab:phis}) while the sampling latent space improves performance in MiniGrid as well (Table \ref{tab:phih}). 

\begin{table*}[!htb]
\caption{\small Results for \ours{} in MiniGrid compared to \ours{} with the use of a sampling latent space. While the differences are small on most cases, \ours{} with a latent $\Phi_h$ is better on all tasks. 256 trials per method.}
\footnotesize
\setlength\tabcolsep{3pt}
\label{tab:phih}
\centering
\begin{tabular}{p{2.5cm}p{1.5cm}p{1.5cm}p{1.5cm}p{1.5cm}p{2cm}p{2cm}}
% {@{}lcccccccccccc@{}}
\toprule
& \textbf{DK-6} & \textbf{DK-8} & \textbf{KC-S3R3} & \textbf{KC-S3R4} & \textbf{MR-N4S5} & \textbf{MR-N6}\\ 
\midrule
% RS & \textbf{0.97}$\pm$\textbf{0.01} & 0.34$\pm$0.13 & -2.38$\pm$0.09 & \textbf{-4.27$\pm$0.12} & -18.16$\pm$0.80 & -119.39$\pm$4.64 \\
% CEM & 0.03$\pm$0.12 & -3.34$\pm$0.34 & -3.40$\pm$0.08 & -4.93$\pm$0.13 & -22.88$\pm$1.00 & -131.32$\pm$5.24\\
% CMAES & 0.93$\pm$0.03 & 0.23$\pm$0.14 & -2.46$\pm$0.09 & -4.44$\pm$0.12 & -14.31$\pm$0.78 & -117.50$\pm$4.61 \\
\ours{}-$\mathrm{latent}\Phi_h$ & 0.97$\pm$0.02 & 0.48$\pm$0.11 & -2.19$\pm$0.15 & -4.22$\pm$0.13 & -10.68$\pm$0.68 & -112.72$\pm$4.46 \\
\ours{} & 0.95$\pm$0.03 & 0.46$\pm$0.13 & -2.27$\pm$0.09 & -4.37$\pm$0.13 & -11.68$\pm$0.75 & -113.53$\pm$4.49  \\

% ours state & & & & -2.27$\pm$1.48 & -4.37$\pm$2.08\\
% ours Slatent & & & & -2.19$\pm$2.47 & -4.22$\pm$2.10 \\
\bottomrule 
\end{tabular}
\end{table*} 

We additionally ablate on the latent space (used for both partitioning and sampling) in the easiest of our molecular design tasks, the QED property. Specifically, we build the molecular SMILES string autoregressively, using a discrete action space with 10 choices: the 9 most common characters in molecular SMILES strings, in addition to an end token. (We limit the space of possible characters in order to increase the chances of generating well-formed SMILES strings.) We optimize in a continuous space of action probabilities as in MiniGrid, and allow a maximum length of 50 characters.

The poor results demonstrate the absolute \textit{necessity} of a latent space in the molecular design task (Table \ref{tab:molecule_nolatent}). While typical molecular SMILES strings for this task are 30 to 50 characters long, both \ours{} and baselines struggle to generate well-formed strings even of length 3 to 5 without the pre-trained latent space. Accordingly, the performance is drastically lower for all methods. 

\begin{table*}[htbp]
% \caption{Results for molecules}
\footnotesize
\setlength\tabcolsep{3.5pt}
\centering
\begin{tabular}
% {p{2cm}p{2cm}p{2cm}p{2cm}p{2cm}p{2cm}p{2cm}}
{lcccc}
\toprule
& \textbf{QED} \\
\midrule
% & \textbf{QED, No Latent} \\ 
% \midrule
RS-$\mathrm{no}\Phi$ & 0.417 $\pm$ 0.002 \\
CEM-$\mathrm{no}\Phi$ & 0.411 $\pm$ 0.002  \\
CMA-ES-$\mathrm{no}\Phi$ & 0.403 $\pm$ 0.003 \\
\ours{}-$\mathrm{no}\Phi$ & 0.416 $\pm$ 0.003 & \\
\midrule  
RS & 0.897 $\pm$ 0.001 \\
CEM & 0.906 $\pm$ 0.003  \\
CMA-ES & 0.888 $\pm$ 0.004 \\
\ours{} & 0.916 $\pm$ 0.001 & \\

\bottomrule 
\end{tabular}
\caption{\small Mean and standard deviation across 128 random seeds for \ours{} and baselines on QED, with and without the pre-trained latent space.}
\label{tab:molecule_nolatent}
\end{table*}

% \kevin{TODO, from preliminary experiments in the past it should be similar to the baselines without latent space}

\subsection{Other Inner Solvers}

In this work we have used CMA-ES as the inner solver due to its speed and acceptable performance. The original LaMCTS work used the TuRBO solver~\cite{eriksson2019scalable}, which is prohibitively slow for many of our experiments. Nevertheless, we have run experiments on the MiniWorld tasks using a smaller number of trials to check performance using an alternate inner solver, both on \ours{} and also on the LaMCTS baseline (Table \ref{tab:turbo}. In most cases TuRBO performs equal or worse; we hypothesize this is because our tasks use a smaller query budget per trial compared to the original LaMCTS work, causing TuRBO to use too large a fraction of its total budget in each inner loop. 

\begin{table*}[!htb]
\footnotesize
\setlength\tabcolsep{3.5pt}
\centering
\begin{tabular}{p{3cm}p{2cm}p{2cm}p{2.5cm}p{2cm}p{2cm}p{2cm}}
% {lcccccccccccc}
\toprule
& \textbf{MazeS3}& \textbf{FourRooms}  & \textbf{SelectObj} \\ 
\midrule 
LaMCTS-TuRBO & 21.9 $\pm$ \phantom{0}7.3 & \phantom{0}0.0 $\pm$ \phantom{0}0.0 & \phantom{0}0.0 $\pm$ \phantom{0}0.0 \\
LaMCTS & 23.4 $\pm$ \phantom{0}2.6 & 20.3 $\pm$ \phantom{0}2.5 &\phantom{0}0.8 $\pm$ \phantom{0}0.6 \\
\ours{}-TuRBO & 52.2 $\pm$ 10.4	& 63.6 $\pm$ 14.5	& 47.1 $\pm$ 12.1 \\
\ours{} & 57.0 $\pm$ \phantom{0}3.1 & 89.1 $\pm$ \phantom{0}2.0  & 45.3 $\pm$ \phantom{0}3.1 \\
\bottomrule 
\end{tabular}
\caption{\small Results for \ours{} and LaMCTS with CMA-ES and TuRBO inner solvers in MiniWorld, with fewer trials for TuRBO due to computational expense. TuRBO generally performs equal or worse on this task. }
\label{tab:turbo}
\end{table*}

\section{Baseline Details and Hyperparameter Tuning}\label{sec:hyperparameter_tuning}

\textbf{\ours{}.} For our method, we try $C_p$ in $\{0.5, 1, 2, 4\}$. If the search space is not explicitly bounded, we sample the first $N_{init}$ points used to initialize the partition tree using the same $\sigma$ as CEM. Note $N_{init}$ is not tuned; we use 5 for compiler phase optimization where our query budget is only 50, and 50 elsewhere. No other hyperparameters are tuned. 

\textbf{LaMCTS.} Detailed in Algorithm \ref{alg:plalam}, where we summarize the changes made in \ours{} compared to LaMCTS. It is tuned similarly to our own method \ours{}.

\textbf{RS.} The simplest baseline, in which one simply samples random trajectories and in the end returns the best-performing among them. We do not tune this baseline. 

\textbf{CEM.} An evolutionary method which tracks a population of $N$ samples. At each step, it selects the best $N_e$ samples from its population to initialize the mean $\mu$ for the next generation of $N$ samples, drawn from a Gaussian distribution with standard deviation $\sigma$. However, while too-small $\sigma$ may prevent CEM from escaping local optima, too-large $\sigma$ may yield results little better than random shooting. We find that the choice of $\sigma$ is critical to CEM's performance in our test environments. Therefore, we systematically tune $\sigma$ when running CEM in all environments (checking $\{1, 2, 4, 8\}$). While other parameters such as $N$ and $N_e$ are also tunable, we find that these make a smaller difference, so we did not tune them extensively. 

\textbf{CMA-ES.} A more complex evolutionary method which can be viewed as a variant of CEM. After providing an initial $\mu$ and $\sigma$ for the first generation, CMA-ES determines its own $\sigma$ automatically afterward, while also fitting additional parameters. Even so, we find that its performance is highly sensitive to the initial $\sigma$, and we tune this parameter in the same way that we do for CEM. 

\textbf{VOOT.} An MCTS method which builds a tree on actions. We tune the exploration parameter in the VOO submodule, trying values in $\{0.1, 0.3, 0.5\}$. 
% We tune the exploration parameters of the MCTS tree ($C_p$) as well as the VOO optimizer used as a submodule for the MCTS; we try values $\{0, 1, 2, 4\}$ and $\{0.1, 0.3, 0.5\}$ respectively. In addition, we tune the widening parameter used in adapting MCTS to continuous environments, trying $\{0.1, 0.5, 0.9\}$.

\textbf{RandDOOT.} An MCTS method which builds a tree on actions similar to VOOT, but which splits using axis-aligned boundaries rather than splitting into Voronoi regions; used as a baseline in their original paper \cite{voot}. We did not tune hyperparameters.

\textbf{iLQR.} A gradient-based optimization method for continuously optimizing the planned trajectory, which we run to convergence. It cannot easily escape local optima. As the performance was relatively insensitive to hyperparameters, we did not systematically tune. %As iLQR requires gradients of returns with respect to the action, we run iLQR on continuous-action-space environments only. 

\textbf{PPO.} A standard reinforcement learning algorithm which operates in the parameter space, unlike our other baselines. Since PPO is relatively robust to hyperparameters~\cite{schulman2017proximal}, we didn't systematically tune. 

\section{Additional Environment Details}
\subsection{MiniWorld}\label{sec:miniworld_details}

We modified the original MiniWorld environments to have continuous action spaces and to have more consistent difficulty across random seeds, as follows. 

\textbf{MazeS3}. A 3x3 maze of rooms which are each 3 units by 3 units, with walls between rooms being 0.25 units wide. The maze is constructed by recursive backtracking from the top left room. The agent begins in the top left room and the goal is placed in the last room generated in the maze construction. The step size is 0.3 units and the environment length is 216 steps. The final sparse reward is the Euclidean distance between the agent and the goal if the goal is not reached, otherwise a fixed reward of 1 penalized by a fraction of the number of steps taken, down to a minimum of 0.8. 

\textbf{FourRooms}. A 14x14 unit space with a 6x6 room in each corner. Adjacent rooms are connected by a width-2 corridor along the outer edge of the space, i.e., there is a cross-shaped obstacle in the center. The agent starts in a random location and the goal is in the diametrically opposite location. The step size is 0.2 units and the environment length is 250 steps. The final sparse reward is the Euclidean distance between the agent and the goal if the goal is not reached, otherwise a fixed reward of 1 penalized by a fraction of the number of steps taken, down to a minimum of 0.8. 

\textbf{SelectObj}. A 12x12 unit open space. The agent starts in the center. The near goal is 4 to 4.5 units away and the far goal is 5 to 5.5 units away. The two goals are 3 to 4 units away from each other. The step size is 0.05 units and the environment length is 200 steps. Unlike MazeS3 and FourRooms, SelectObj does not terminate upon reaching a goal. The final sparse reward is the Euclidean distance between the final agent position and the closest goal to the final position, plus a fixed reward of 1 for being within 1 unit of the original far goal. 

\subsection{Compiler Phase Ordering}\label{sec:autophase_details}
The action space consists of 46 different program transformations, and a trajectory consists of 45 transformations (quite short, considering many transformations have no effect unless applied in a specific order). The reward is the difference between the original and final number of execution cycles. Since the environment is deterministic, we only run 1 trial for each method. Thus far we have followed the setup in \cite{haj2020autophase}; however, unlike \cite{haj2020autophase}, we allow a budget of only 50 trajectory queries.

% \section{Code}

% Our code is available at \url{https://github.com/yangkevin2/plalam}. 

\iffalse
\section{List of relevant publications}
List relevant papers:
\begin{enumerate}
    \item \url{https://arxiv.org/pdf/1906.08649.pdf} POPLIN
    \item \url{https://arxiv.org/pdf/1805.12114.pdf} PETS
    \item \url{https://arxiv.org/abs/2007.00708} original LaMCTS
    \item \url{https://arxiv.org/abs/1906.06832} LaNAS
    \item \url{https://arxiv.org/abs/2006.16712} MBRL survey
    \item \url{https://arxiv.org/abs/1802.09081} temporal difference models
    \item latent plans based on searching for paths through a graph \url{http://proceedings.mlr.press/v120/yang20b/yang20b.pdf} \url{https://arxiv.org/abs/1906.05253}
    \item learning latent plans via human-controlled play \url{http://proceedings.mlr.press/v100/lynch20a.html}
    \item \url{https://arxiv.org/abs/2010.02193} DreamerV2
    \item goal-conditioned MBRL
    \url{https://arxiv.org/abs/1807.09341}
    \url{https://arxiv.org/abs/1811.07819}
    \item some other recent MBRL approaches \url{http://proceedings.mlr.press/v119/sekar20a/sekar20a.pdf} \url{http://proceedings.mlr.press/v84/kamthe18a.html}
    \url{https://arxiv.org/abs/2102.08363}
\end{enumerate}
\fi

\end{document}